\documentclass{article}

\usepackage[margin=1in]{geometry}
\usepackage{natbib}

\usepackage[utf8]{inputenc} 
\usepackage[T1]{fontenc}    
\usepackage{hyperref}       
\hypersetup{colorlinks=true,       
    linkcolor=blue,          
    citecolor=blue}
\usepackage{url}            
\usepackage{booktabs}       
\usepackage{amsfonts}       
\usepackage{nicefrac}       
\usepackage{microtype}      
\usepackage{xcolor}         
\usepackage{graphicx}
\usepackage{caption}
\usepackage{subcaption}

\usepackage{amsthm}

\usepackage[english]{babel}

\usepackage[T1]{fontenc}
\usepackage{amsmath}
\usepackage{amssymb}
\usepackage{amsfonts}
\usepackage{multirow}
\usepackage{mathtools}
\usepackage{breqn}
\usepackage{algorithm}
\usepackage[noend]{algpseudocode}
\usepackage{float}
\newfloat{algorithm}{t}{lop}

\usepackage{bbm}
\usepackage{dsfont}
\usepackage{graphicx}
\usepackage{natbib}
\usepackage{cases}
\usepackage[colorinlistoftodos]{todonotes}

\renewcommand{\hat}{\widehat}
\def\shownotes{0}  \ifnum\shownotes=1
\newcommand{\authnote}[2]{[#1:#2]}
\else
\newcommand{\authnote}[2]{}
\fi

\renewcommand{\paragraph}[1]{\smallskip\noindent\textbf{#1}\;}

\newtheorem{theorem}{Theorem}[section]
\newtheorem{assumption}[theorem]{Assumption}
\newtheorem{proposition}[theorem]{Proposition}
\newtheorem{lemma}[theorem]{Lemma}
\newtheorem{definition}[theorem]{Definition}
\newtheorem{corollary}[theorem]{Corollary}

\newtheorem{remark}[theorem]{Remark}

\newtheorem*{remark*}{Remark}

\newtheorem*{observation*}{Observation}

\numberwithin{equation}{section}

\newcommand{\ds}{d_s}
\newcommand{\E}{\mathbb{E}}

\newcommand{\R}{\mathbb{R}}

\newcommand{\cE}{\mathcal{E}}

\newcommand{\cH}{\mathcal{H}}

\newcommand{\cN}{\mathcal{N}}
\newcommand{\cP}{\mathcal{P}}
\newcommand{\cQ}{\mathcal{Q}}

\newcommand{\cX}{\mathcal{X}}
\newcommand{\cY}{\mathcal{Y}}

\DeclareMathOperator*{\argmin}{arg\,min}

\newcommand{\Gnorm}[1]{{\left\vert\kern-0.25ex\left\vert\kern-0.25ex\left\vert #1 
		\right\vert\kern-0.25ex\right\vert\kern-0.25ex\right\vert}}
\newcommand{\gnorm}[1]{{\vert\kern-0.25ex\vert\kern-0.25ex\vert #1 
		\vert\kern-0.25ex\vert\kern-0.25ex\vert}}

\DeclareMathOperator{\sign}{\textup{sgn}}

\title{Iterative Feature Matching: \\Toward Provable Domain Generalization \\with Logarithmic Environments}

\author{
  Yining Chen\\Stanford University\\\texttt{cynnjjs@stanford.edu}
  \and Elan Rosenfeld\\Carnegie Mellon University\\\texttt{elan@cmu.edu}
  \and Mark Sellke\\Stanford University\\\texttt{msellke@stanford.edu}
  \and Tengyu Ma\\Stanford University\\\texttt{tengyuma@stanford.edu}
  \and Andrej Risteski\\Carnegie Mellon University\\\texttt{aristesk@andrew.cmu.edu}
}

\date{}

\begin{document}
\maketitle

\begin{abstract}
Domain generalization aims at performing well on unseen test environments with data from a limited number of training environments. Despite a proliferation of proposal algorithms for this task, assessing their performance both theoretically and empirically is still very challenging. Distributional matching algorithms such as (Conditional) Domain Adversarial Networks~\citep{ganin2016domain, NEURIPS2018_ab88b157} are popular and enjoy empirical success, but they lack formal guarantees. Other approaches such as Invariant Risk Minimization (IRM) require a prohibitively large number of training environments --- linear in the dimension of the spurious feature space $\ds$--- even on simple data models like the one proposed by~\citet{rosenfeld2021the}. Under a variant of this model, we show that both ERM and IRM cannot generalize with $o(\ds)$ environments. We then present an iterative feature matching algorithm that is guaranteed with high probability to yield a predictor that generalizes after seeing only $O(\log \ds)$ environments. Our results provide the first theoretical justification for a family of distribution-matching algorithms widely used in practice under a concrete nontrivial data model.
\end{abstract}

\section{Introduction}

Domain generalization aims at performing well on unseen environments using labeled data from a limited number of training environments~\citep{blanchard2011generalizing}. In contrast to transfer learning or domain adaptation, domain generalization assumes that neither labeled or unlabeled data from the test environments is available at training time. For example, a medical diagnostic system may have access to training datasets from only a few hospitals, but will be deployed on test cases from many other hospitals~\citep{choudhary2020advancing}; a traffic scene semantic segmentation system may be trained on data from some specific weather conditions, but will need to perform well under other conditions~\citep{yue2019domain}.

In the empirical literature, invariance of a ``signal'' feature distribution conditioned on the label, i.e. $P(\Phi(x) \mid y)$, is the underlying assumption in widely adopted algorithms such as Correlation Alignment (CORAL)~\citep{CORAL, sun2016deep}, Maximum Mean Discrepancy (MMD, \citet{gretton2012kernel})~\citep{li2018domain}, and (Conditional) Domain Adversarial Networks~\citep{ganin2016domain, NEURIPS2018_ab88b157}. These algorithms are popular and enjoy empirical success in both domain adaptation and generalization, but they lack formal guarantees. Previous empirical works usually justify these algorithms using the generalization bounds based on $\cH$-divergence~\citep{ben2010theory}, but those bounds are vacuous for modern settings of practical interest and thus cannot explain their success. To obtain concrete guarantees, it is necessary to study data models that encode structure reflective of settings of interest. Prior works attempting to theoretically characterize the performance of feature matching algorithms emphasize lower bounds~\citep{zhao2019learning, tachet2020domain}. In this work, we seek to give the first positive theoretical justification for feature distribution matching algorithms.

The other major alternative assumption in the literature is enforcing invariance of the label distribution conditioned on the signal features. Invariant Risk Minimization (IRM)~\citep{arjovsky2019invariant} assumes $\E[y\mid\Phi(x)]$ is invariant, and follow-up works assume invariance of higher moments~\citep{xie2020risk,jin2020domain, mahajan2020domain, krueger2020out, bellot2020generalization}.
However, empirical results for these algorithms are mixed:~\citet{gulrajani2021in, aubin2021linear} present experimental evidence that these methods do not consistently outperform ERM for either realistic or simple linear data models, when fairly evaluated. 

Recent theoretical works~\citep{rosenfeld2021the, pmlr-v130-kamath21a} also question the theoretical foundations of IRM and its variants, shedding light on their failure conditions. These works study specific data generative models; a common assumption is that, conditioned on the label, some \emph{invariant features} have identical distribution for all environments, and other \emph{spurious features} have varying distributions across environments. The goal of domain generalization is then to obtain an \emph{invariant predictor}, i.e. a classifier which uses only the invariant features. These works also often assume each training environment contains infinite samples.
Thus, the central measure of domain generalization is the number of \emph{environments} needed to recover an invariant predictor---we refer to this measure as the \emph{environment complexity} of a learning algorithm. \citet{rosenfeld2021the} prove that even for a simple generative model and linear classifiers, the environment complexity of IRM---and other objectives based on the same principle of invariance---is at least as large as the dimension of the spurious latent features, $\ds$. Further results by~\citet{pmlr-v130-kamath21a, ahuja2021empirical} also point to a linear environment complexity. Although the models in these works are simple, they help elucidate why existing algorithms fail and can therefore help inform better algorithmic design.

IRM's linear environment complexity is prohibitive for realistic applications. We usually expect there to be many more spurious dimensions than signal dimensions, whereas the number of environments observed is presumed to be much fewer. In this paper, we show that a variant of distributional matching algorithm (Althorithm~\ref{alg}) generalizes with sublinear ---even logarithmic environment complexity.

We study a more natural ``smoothed covariance'' extension of the data model in \citet{rosenfeld2021the}. Instead of assuming that the spurious features have isotropic covariances, we model their covariances as generic random positive definite matrices with adversarial biases, i.e., the covariances can be arbitrary with added noise. Under this new model, 
we show that ERM and IRM still do not generalize after seeing fewer than $\ds$ environments (Theorem~\ref{thm:ERM}, Theorem~\ref{thm:irm}).
On the other hand, we show that a conceptually simple algorithm based on iterative feature matching (IFM) is guaranteed with high probability to recover only the invariant features with environment complexity $E = O(\log \ds)$.
Our method therefore \emph{provably} achieves generalization to the worst-case test environment with a much more reasonable number of observed environments (Theorem~\ref{thm:main_wrapper}). 

The main idea behind IFM is to iteratively project the features to a lower dimension, in each round matching the label-conditioned feature distributions on a small, disjoint subset of the training environments. As a projection which induces invariance in the non-invariant features across one subset of environments is unlikely to do so for a different subset, we can show that with high probability, each projection removes only spurious feature dimensions. By avoiding an end-to-end training scheme, we effectively prevent the different environments from ``colluding'' to create a misleading solution which which depends on spurious features.
As a result, IFM recovers optimal invariant predictor after $O(\log \ds)$ iterations, requiring $O(\log \ds)$ environments. 

To corroborate the advantages of the proposed method, we perform experiments on a Gaussian dataset and a semi-synthetic \emph{Noised MNIST}~\citep{mnist} dataset, where the background noise spuriously correlates with the label. Our results suggest that practitioners may benefit from feature matching algorithms when the distinguishing property of the signal feature is indeed conditional distributional invariance, and may get additional advantage via matching at multiple layers with diminishing dimensions, echoing existing empirical observations~\citep{pmlr-v37-long15, NIPS2017_a8baa565}.

\subsection{Additional related works}

\paragraph{Broader theoretical study of domain generalization.}
Other works analyze the task of domain generalization more generally in different settings. \citet{blanchard2011generalizing, muandet2013domain} assume a fixed prior over environments and present classification algorithms with generalization bounds that depend on properties of the prior. 
Considering instead convex combinations of domain likelihoods, \citet{albuquerque2019generalizing} give a generalization bound for distributions with sufficiently small $\mathcal{H}$-divergence, while \citet{rosenfeld2021online} model domain generalization as an online game, showing that generalizing beyond the convex hull is NP-hard.
\section{Preliminaries}
\subsection{Domain generalization}
 
In domain generalization, we are given a set of $E$ training environments $\cE_{tr}$ indexed by $e \in [E]$,\footnote{We define $[n]=\{1, \dots, n\}$; $\mathbf{0}_{n \times m} \in \R^{n \times m}$ denotes an all-zero matrix; $\mathbb{S}^{d}$ is the unit sphere in $\R^{d+1}$; $\sign{(c)} \in \{\pm 1, 0\}$ is the sign of scalar $c \in \R$. $\dagger$ denotes the Moore-Penrose pseudo-inverse.}
and a set of test environments $\cE_{ts}$. For environment $e$ we have $n$ examples $\{(X_i^e, Y_i^e)\}_{i=1}^n$ drawn from the distribution $P_e$. In this work we study the infinite sample limit $n \rightarrow \infty$ so as to separate the effect of limited training environments from that of limited samples \emph{per} environment, as is done in previous theoretical works~\citep{rosenfeld2021the, pmlr-v130-kamath21a}.
Let $\cX$, $\cP$, $\cY$ denote the space of inputs, intermediate features, and labels. For a featurizer $\Phi: \cX \rightarrow \cP$ and classifier $w: \cP \rightarrow \cY$, their risk on environment $e$ is denoted by $R_{\Phi,w}^e = \E_{(X, Y) \sim P_e}[l( w \circ \Phi(X), Y)]$ for any common loss function $l$. In this paper we focus on $\cY=\{\pm 1\}$, linear featurizers $\Phi(X)=U X$ for $U \in \R^{k \times d}$, and unit-norm predictors $\hat{Y}=\sign{(w^\top U x)}$ where $w \in \R^k$ and $\|w^\top U\|_2=1$ for some feature dimension $k \le d$ chosen by the algorithm. A predictor's 0-1 risk on environment $e$ is denoted by $R_{U, w}^e = \Pr_{(X, Y) \sim P_e}[\sign{(w^\top U X)} \ne Y]$. We focus on unit-norm predictors because we evaluate on the 0-1 risk on test environments, which are invariant to the scaling of $w^\top U$ under our data model. 

\subsection{Baseline algorithms}

We analyze the performance of our proposed method and compare it to two baseline algorithms, ERM and IRM. 

ERM learns a classifier that minimizes the average loss over all training environments, where $l$ is any common training loss such as the logistic loss: 
\begin{align*}
    \min_{w \in \mathbb{S}^{d-1}} \frac{1}{E} \sum_{e \in [E]} \E_{(X, Y) \sim P_e}[l( w^\top X, Y)].
\end{align*}
IRM learns a featurizer $\Phi(X) \in \R^k$ such that the optimal classifier on top of the featurizer is invariant across training environments. As we focus on linear classifier, it is equivalent to learning a linear transformation $U \in \R^{k \times d}$ such that it induces a classifier $w$ that is optimal for all $e \in \E_{tr}$:

\begin{align*}
\min_{U \in \R^{k \times d}, w \in \R^{k}, \|w^\top U\|_2=1} \frac{1}{E} \sum_{e \in [E]} \E_{(X, Y) \sim P_e} l((w^\top U X), Y) \\
\textup{s.t.}\quad w \in \argmin_{w'\in \R^k}{\E_{(X, Y) \sim P_e}[l(({w'}^\top U X), Y)]}, \forall e \in \cE_{tr}.
\end{align*}

Note that this is objective is \emph{not} the same as feature distribution matching; IRM only tries to match the first moment. Observe that this constrained objective is intended to solve a minimax domain generalization problem, as opposed to ERM which is typically viewed as minimizing the risk in expectation.

\section{Problem setup}
\label{sec:problem-setup}

We first recall the data model from \citet{rosenfeld2021the}. We assume without loss of generality that the label $Y$ is uniformly randomly drawn from $\{\pm 1\}$ (extension of our theorems to $Y=1$ with probability $\eta \ne 0.5$ is straightforward). Latent variable $Z$ consists of invariant features $Z_1 \in \R^r$ and spurious features $Z_2 \in \R^{d_s}$ where $d_s = d-r$. The number of spurious features can be much larger than the number of invariant features, i.e. $d_s \gg r$. The input $X \in \R^d$ is generated via a linear transformation of latent variable $Z$, i.e. $X=S Z$ for a matrix $S \in \R^{d \times d}$ such that its left $r$ columns have rank $r$ (so that there are $r$ invariant dimensions). 

For each training environment indexed by $e \in [E]$, the invariant features conditioned on $Y$ are drawn from a Gaussian distribution with mean $Y \cdot \mu_1 \in \R^r$ and nonsingular covariance $\Sigma_1 \in \R^{r \times r}$. The spurious features conditioned on $Y$ have mean $Y \cdot \mu_2^e \in \R^{d_e}$ and covariance $\Sigma_2^e \in \R^{d_e \times d_e}$ where $\mu_e$'s and $\Sigma_e$'s vary across $e \in [E]$. The assumption of symmetric class center with respect to the origin can also be relaxed.
Define $\mu^e = [\mu_1, \mu_2^e]$ and $\Sigma^e = [\Sigma_1, \mathbf{0}_{r \times \ds}; \mathbf{0}_{\ds \times r}, \Sigma_2^e]$. The overall data model for training environments is summarized below:
\begin{align*}
    Y &\sim \textup{unif}\{\pm1\} \\
    Z_1 | Y &\sim N(Y \cdot \mu_1, \Sigma_1) \in \R^r \\
    Z_2 | Y &\sim N(Y \cdot \mu_2^e, \Sigma_2^e) \in \R^{d_s} \\ Z&=[Z_1, Z_2] \in \R^d \\
    X &= S Z.
\end{align*}
Since the goal of invariant feature learning is to learn a predictor that only uses the invariant features, one reasonable measure for domain generalization is a predictor's performance on test environments where the spurious features $Z_2$ are drawn from a different distribution---in particular, they are usually chosen adversarially. A classifier that predicts using the spurious features will perform badly on such test environments. When modeling the test environments, we consider the difficult scenario where there is one corresponding test environment for each training environment, whose parameters are the same except that the spurious means are flipped. Formally, for each environment $e \in \cE_{tr}$ we construct a corresponding test environment $e'
\in \cE_{test}$ where
\begin{align*}
    Z_2 &\sim N(-Y \cdot \mu_2^e, \Sigma_2^e) \in \R^{d_s}.
\end{align*}

In this setting where the observations $X$ are a linear function of the latents $Z$, \citet{rosenfeld2021the} assume that the covariances of spurious features are isotropic and vary only in magnitude:
\begin{assumption}[Data model for spurious covariances in~\citet{rosenfeld2021the}]\hfill
\label{ass:old_cov}
\begin{center}
$\Sigma_2^e = \sigma_e^2 I_{d_s}$, where $\sigma_e$ is a scalar for an environment indexed by $e$.
\end{center}
\end{assumption}

We consider a generalized model where the covariances of spurious features for each environment is a generic random PSD matrix, instead of only random in scaling:
\begin{assumption}[Data model for spurious covariances in this work]\hfill
\label{ass:new_cov}

\begin{center}
    $\Sigma_2^e \sim \overline{\Sigma_2^e} + G_e G_e^\top$, where $\overline{\Sigma_2^e} \in \R^{\ds \times \ds}$ is arbitrary (and can be adversarial),\\ and $[G_e]_{i,j} \stackrel{iid}{\sim} N(0, 1)$ for all $i, j \in [\ds]$.\\ Furthermore,$\max_e{\|\overline{\Sigma_2^e}\|_2^2} \le D$.
\end{center}
\end{assumption}

This form of assumption is common in smoothed analysis~\citep{10.1145/990308.990310}. In the next section, we show that baseline algorithms ERM and IRM still have $o(\ds)$ environment complexity for under assumption~\ref{ass:new_cov}, whereas an iterative feature matching algorithm (Algorithm~\ref{alg}) requires only $O(\log{\ds})$ training environments. Note that the environment complexity of our algorithm only depends logarithmically on the norm bound $D$.
\section{Main results}
Armed with Assumption~\ref{ass:new_cov}, we are now ready to present our main results. We begin by presenting our algorithm based on iterative feature matching. In the following subsections, we provide formal guarantees for its environment complexity and compare it to ERM and IRM.

\subsection{Iterative feature matching algorithm}
\label{subsec:alg}
\begin{algorithm}
\caption{Iterative Feature Matching (IFM) algorithm \label{alg}}
\begin{algorithmic}[1]
\Require Invariant feature dimension $r$, total feature dimension $d$, 
number of training environments $E=|\cE_{tr}|$, infinite samples $\{(X_i^e,Y_i^e)\}_{i=1}^\infty \sim P_e$ from each environment $e \in \cE_{tr}$.
\State Let $\{\cE_t\}_{t=1}^T$ be a partition of $\cE_{tr}$.
\State $r_0 \leftarrow d$, $t \leftarrow 0$.
\While {$r_t > r$}
    \State $t \leftarrow t+1$.
    \State Find the maximum dimension $r_t$ such that there exists orthonormal $U_t \in \R^{r_t \times r_{t-1}}$ and $C_t \in \R^{r_t \times r_t}$, where for all $e \in \cE_t$,
    \begin{align}
        \E_{(X,Y) \sim P_e}[U_t \dots U_1 X X^\top U_1^\top \dots U_t^\top  | Y]=C_t.
        \label{eq:match}
    \end{align}
\EndWhile
\State Return a classifier on projected features that minimizes the average risk

$
    \hat{w} = \min_{w \in \mathbb{S}^{r-1}} \frac{1}{E} \sum_{e \in [E]} \E_{(X,Y) \sim P_e} l(U_t \dots U_1 X, Y)$.
\end{algorithmic}
\end{algorithm}

We hope to recover the invariant features by imposing constraints that are satisfied by them but not the spurious ones. A natural idea is to match the feature means and covariances across $\cE_{tr}$. Since $\mu_1, \Sigma_1$ are constant, any orthonormal featurizer $U \in \R^{r \times d}$ such that $U S$ has only non-zero entries in the first $r$ rows yields invariant means $U S \mu^e$ and covariances $U S \Sigma^e S^\top U^\top$. Thus we need $E$ large enough such that any $U' \in \R^{r \times d}$ using spurious dimensions cannot match the means and covariances. Informally, for each $e \in [E]$ we get $r \times r$ equations from matching covariances $U \Sigma^e U^\top = C$, and we have $r \times d$ parameters to estimate in $U$. Rough parameter counting suggests that if we match covariances of all $E$ environments jointly, we need at least $E>d/r$ environments to find a unique solution. Our key observation is that, due to the independence of randomness in $\Sigma_2^e$, we can split $E$ environments into disjoint groups $\cE_1, \dots, \cE_T$, and use $\cE_t$ to train an orthonormal featurizer that shrinks the feature dimensions from $r_{t-1}$ to $r_{t}$. Thus, in each round we shrink the dimension by a constant factor using a constant number of environments. The main theoretical challenge that remains is to show that in each iteration, with high probability, \emph{only} spurious features are projected out.

This brings us to IFM (Algorithm~\ref{alg}), which proceeds in $T=O(\log{d_s})$ rounds. Starting with an input dimension $r_0=d$, each round we learn an orthonormal matrix $U_t$ projecting features from $r_{t-1}$ to $r_t$ dimensions so that the feature covariances after projection match across a fresh set of training environments. To ensure that all invariant dimensions are preserved, we always find the projection with the maximum possible dimension $r_t$ that still matches the covariances, until we are left with only $r$ dimensions. 

CORAL~\citep{CORAL} also matches feature means and covariances. However, there are several salient differences between IFM
and CORAL:  first, CORAL does not enforce that the featurizer is orthonormal; second, IFM learns to extract features in an unsupervised manner, whereas CORAL jointly minimizes the supervised loss and feature distribution discrepancy; third, IFM matches the feature distributions at multiple layers and uses a disjoint set of environments for each layer---this iterative process is necessary for the theoretical guarantees we provide. Despite these differences, our theoretical results serve as a justification for using feature matching algorithms in general, when the distinguishing attribute of signal vs. spurious features is that the former have invariant distributions across all environments. In section~\ref{sec:experiments} we empirically evaluate whether bridging the gap between IFM and CORAL may improve test accuracy.

The following theorem states that the environment complexity of IFM is logarithmic in the spurious feature dimension. A proof sketch is given in Section~\ref{sec:proof_sketch}.

\begin{theorem}[IFM upper bound]
\label{thm:main_wrapper}
Under assumption~\ref{ass:new_cov}, suppose at each round IFM uses $|\cE_t|=\Tilde{\Omega}(1)$ \footnote{$\Tilde{\Omega}(\cdot)$ hides logarithmic factors in $D, r, \ds$.} training environments, with probability $1-\exp{(-\Omega(d_s))}$, IFM terminates in $O(\log{d_s})$ rounds, and outputs $\hat{w}=w^*$.
\end{theorem}

By way of remarks, we note that there is in fact a simple algorithm that achieves $O(1)$ environment complexity under assumption~\ref{ass:new_cov} (see Appendix~\ref{sec:simple_app}). However, this algorithm is extremely brittle and reliant on the data model, and cannot be extended to other settings, whereas feature matching algorithms can be applied to general model architecture and realistic datasets. The goal of this paper is to provide theoretical justification for distribution matching matching algorithms.

\subsection{ERM and IRM still have linear environment complexity}
\label{ref:subsec:ERMIRM}
In the previous section, we showed that IFM has low environment complexity thanks to the additional structure we imposed in our model. However, we haven't excluded the possibility that this additional structure also allows ERM and IRM to succeed. These next two results demonstrate that this is not the case.

\paragraph{ERM has low test accuracy}
In contrast to IFM, ERM still suffers from linear environment complexity under Assumption~\ref{ass:new_cov}. The first theorem says there are hard instances where the ERM solution has worse-than-random performance on the test environments.

\begin{theorem}[ERM lower bound]
\label{thm:ERM}
Suppose $E \le \ds$, parameters $\mu_1 \in \R^r$, $\Sigma_1 = \sigma_1^2 I_r$, $\mu_2^e \in \R^{\ds}$, $\overline{\Sigma_2^e} = \sigma_2^2 I_{\ds}$ (recall $\Sigma_2^e \sim \overline{\Sigma_2^e} + G_e G_e^\top$).
Then any unit-norm linear classifier which achieves accuracy $\ge \Phi\left(\frac{2\|\mu_1\|}{\min(\sigma_1,\sigma_2)}\right)$ on all training environments will suffer 0-1 error at least $\frac{1}{2}$ on every test environment with flipped spurious mean, where $\Phi$ is the standard Normal CDF.
\end{theorem}
A complete proof of Theorem~\ref{thm:ERM} is in Appendix~\ref{sec:ERM_proof}. Note that it is quite reasonable to assume that the ERM solution satisfies the accuracy condition.
In particular, it is common to model the spurious features as having much greater magnitude than the invariant features, since they have much greater dimensionality. For example, with a unit-norm mean we would expect $\|\mu_1\|^2 \approx r/d,\;\|\mu_2^e\|^2 \approx \ds/d$. Then for $r \ll d$ and $\sigma_1,\sigma_2= \omega(1/\sqrt{d})$ one can verify that $2\|\mu_1\|/\min(\sigma_1,\sigma_2)= o(\sqrt{r/d})$ is very close to 0, meaning the lower bound $\Phi\left(\frac{2\|\mu_1\|}{\min(\sigma_1,\sigma_2)}\right)$ is only slightly larger than $\frac{1}{2}$.

\paragraph{IRM fails to learn invariant features}
Our next theorem proves that even under Assumption~\ref{ass:new_cov}, IRM is still not guaranteed to find an invariant predictor. We can show this by proving that when $E \le d_s$, we can find a featurizer that only uses spurious dimensions, i.e., $u_s \in \R^{d_s}$, such that $u_s^\top \Sigma_2^e u_s = u_s^\top \mu_2^e$ for all $e \in \cE_{tr}$. If so, the optimal predictor on top of features $u_s^\top Z_2$, $\hat{w^e} = (u_s^\top \Sigma_2^e u_s)^{-1} u_s^\top \mu^e$ is invariant across all $e$, and can therefore be the preferred solution IRM when the spurious features have larger margin on the training environments.

\begin{theorem}[IRM lower bound]
\label{thm:irm}
Suppose $E \le \ds$. If $\mu_2^1,\dots, \mu_2^E\in\mathbb R^{\ds}$ are linearly independent, then there exists $u_s \in \R^{d_s}$, $\|u_s\|_2 > 0$, such that $u_s^\top \Sigma_2^e u_s = u_s^\top \mu_2^e$ for all $e \in [E]$.
\end{theorem}

Observe that each environment provides an ellipsoidal constraint $E_{e}=\{u_s \in\mathbb R^\ds: u_s^{\top} \Sigma_2^e u_s-u_2^\top \mu_2^e=0\}$. The origin is a trivial intersection. We prove the existence of a non-trivial intersection using tools from differential topology. The key lemma is that the total number of intersection points between two manifolds of complementary dimensions $k,d-k$ is even when certain tranversality conditions hold. Using these techniques, we show that $\left|\bigcap_{e} E_e\right|\geq 2$ for almost all matrices $\Sigma_2^1, \dots, \Sigma_2^E$, as long as the means are linearly independent. A complete proof of Theorem~\ref{thm:irm} is in Appendix~\ref{sec:irm_proof}.

\section{Proof sketch for the main upper bound Theorem~\ref{thm:main_wrapper}}
\label{sec:proof_sketch}

To argue that IFM outputs a featurizer $U_1\dots U_T$ that does not use the spurious features, we need to show that the right $\ds$ columns of matrix $U_T \dots U_1 S$ are all-zero. The main lemma below says that this happens with high probability if we match $\Tilde{\Omega}(1)$  environments at every iteration,

\begin{lemma}

\label{thm:upper_ours}

If for all $1 \le t \le T$,
$|\cE_t|=E_t = \Omega\left(c (\log{(D/\ds)}+\log{d_s}) \right)$, and $U_1, \dots, U_T$ are the orthonormal matrices returned by IFM, then with probability $1-\exp{(-\Omega(d_s))}$, $r_t-r < (r_{t-1}-r+1) / c$ for all $t$, and if we write $U_T \dots U_1 S = [A, B]$, where $B \in \R^{r \times d_s}$, then $B=\mathbf{0}_{r \times d_s}$.
\end{lemma}

Theorem~\ref{thm:main_wrapper} follows from Lemma~\ref{thm:upper_ours} as follows: We take $c=2$ and the algorithm terminates in $T=O(\log{d_s})$ rounds. Therefore with an environment complexity of $O(\log{d_s})$, we learn a feature extractor $U = U_T \dots U_1$ that does not use any spurious dimensions. Since $U$ is orthonormal, it must contain all signal dimensions. The predictor on top of this representation uses all and only signal dimensions, so with high probability, IFM outputs $\hat{w}=w^*$.
 
The first step towards proving Lemma~\ref{thm:upper_ours} is to show that with high probability, any \textit{one-layer} featurizer $Q_1 \in \R^{k_1 \times \ds}$ that uses \emph{only} spurious dimensions cannot match feature covariances from $\tilde{\Omega}(\ds/k_1)$ environments. If a featurizer $U_1 \in \R^{r_1 \times d}$ uses $k_1$ spurious dimensions, there is a corresponding rank-$k_1$ featurizer $Q_1 \in \R^{k_1 \times \ds}$ that uses \emph{only} spurious dimensions. So Lemma~\ref{lem:only_spu} implies that any $U_1$ that matches covariances in $\cE_1$ must use at most $\ds/E_1$ spurious dimensions. We will then apply this argument recursively until we have 0 spurious dimensions.

\begin{lemma}[Informal version of Lemma~\ref{lem:main}]
\label{lem:only_spu}
For any integer $2 \le k_1 \le \ds/2$, when
$ |\cE_1|= E_1 = \Omega\left( \frac{\ds-k_1}{k_1-1} \max\left\{1, \log\left(\frac{D}{ (k_1-1) \ds}\right), \log\left(\frac{\ds}{ k_1-1}\right)\right\} \right)
$, with probability $1-O( \exp{(-\ds)})$, no orthonormal $Q \in \R^{k_1 \times \ds}$  satisfies that for some constant  $C_1 \in \R^{k_1 \times k_1}$.
\begin{align}
\forall e \in [E_1],\quad Q \Sigma_2^e Q^\top=C_1.
\label{eq:match_Q}
\end{align}
\end{lemma}
The formal statement Lemma~\ref{lem:main} and its proof can be found in Appendix~\ref{sec:upper_proof}. On a high level, we discretize over the space of $Q$, and show that for fixed $Q$, denoting by $q_i$ its $i$-th row, the probability that $q_i^\top G_1 G_1^\top q_j - q_i^\top G_2 G_2^\top q_j=0$ for all $i \ne j$ is small, so Equation~\eqref{eq:match_Q} cannot be true for this fixed $Q$. Union bounding over the covering, with high probability, no $Q$ can satisfy Equation~\eqref{eq:match_Q}.

The next claim says that Lemma~\ref{lem:only_spu} can be applied iteratively, i.e. fixing a featurizer from previous iterations that uses $k_{t-1}$ spurious dimensions, with high probability, any $U_t$ that matches features from $\Omega\left(k_{t-1}/k_t\right)$ new environments uses at most $k_t$ spurious dimensions.
\begin{corollary}[Informal version of Corollary~\ref{cor:recursive}]
\label{cor:recursive_informal}
Suppose $2 \le k_t \le k_{t-1}/2 \le \ds/2$. When
$ |\cE_t|=E_t = \Omega\left(\frac{k_{t-1}-k_t}{k_t-1} \max\left\{1, \log\left(\frac{D}{ (k_t-1) \ds}\right), \log\left(\frac{\ds}{ k_t-1}\right)\right\}\right)$,
for fixed orthonormal $P \in \R^{k_{t-1} \times \ds}$, with probability $1- O(\exp{(-\ds)})$, no orthonormal $Q \in \R^{k_t \times k_{t-1}}$ satisfies $\forall e \in [E_t]$, $Q P \Sigma_2^e P^\top Q^\top=C_t$ for some constant $C_t \in \R^{r_t \times r_t}$.
\end{corollary}
The formal statement Corollary~\ref{cor:recursive} and its proof can be found in Appendix~\ref{sec:upper_proof}. Lemma~\ref{thm:upper_ours} follows from iterative application of Corollary~\ref{cor:recursive_informal}, as shown below.
\begin{proof}[Proof of Lemma~\ref{thm:upper_ours}]
We shall prove that for all $t < T$, if we write $U_t S = [A_t, B_t]$ where $A_t, B_t$ are the left $r$ and right $\ds$ columns of $U_t S$, then $rank((I-P_{B_t})A_t)=r$ and $k_t = rank(B_t) = r_t-r < (r_{t-1}-r+1)/c$ with probability $1- O(t \exp(-\ds))$. Since $T=O(\ds)$,  for $t=T-1$, the probability $1-  O(T\exp(-\ds)) = 1- \exp{(-\Omega(\ds))}$.

Lemma~\ref{lem:preserve} in Appendix~\ref{sec:upper_proof} says that if $rank((I-P_{B_t})A_t)<r$, then we can construct orthonormal $U'_t$ with higher dimension $r'_t>r_t$ that still matches the covariances for environments $\cE_t$. Hence IFM always finds $U_t$ with $rank((I-P_{B_t})A_t)=r$.

To show that the number of spurious dimension decreases, we prove by induction on $t$. For the base case $t=1$, Lemma~\ref{lem:only_spu} says $r_1-r \le k_1  < (\ds-r+1)/c$ with probability $1- O( \exp(-\ds))$. For $t \ge 2$, suppose to the contrary that there is orthonormal $U_t \in \R^{r_t \times r_{t-1}}$ satisfying~\eqref{eq:match} such that $U_t \dots U_1 S = [A_t, B_t]$ where $B_t \in \R^{r_t \times \ds}$, and $rank(B_t) = k_t > (r_{t-1}-r+1)/c$. By induction hypothesis, with probability $1-O((t-1)\exp{(-\ds)})$, we can write $U_{t-1} \dots U_1 S = [A_{t-1}, B_{t-1}]$ where $B_{t-1} \in \R^{r_{t-1} \times \ds}$ has rank $k_{t-1} \ge r_{t-1}-r$. Below we condition on this event.

Writing $B_{t-1}$ in terms of the compact SVD, we get $B_{t-1} = P_{t-1} \Lambda_{t-1} Q_{t-1}$, where $Q_{t-1} \in \R^{k_{t-1} \times \ds}$. Therefore
\begin{align*}
    U_t [A_{t-1}, B_{t-1}] \begin{bmatrix}
    \Sigma_1 & 0\\
    0 & \Sigma_2^e 
    \end{bmatrix} [A_{t-1}, B_{t-1}]^\top U_t^T = C_t \\ \implies U_t B_{t-1} \Sigma_2^e B_{t-1}^\top U_t^\top = C_t'.
\end{align*}
Writing $U_t P_{t-1} \Lambda_{t-1}$ in terms of its compact SVD, we get $U_t P_{t-1} \Lambda_{t-1} = P_t \Lambda_t Q_t$ where $\Lambda_t$ has $k^*$ non-zero singular values and $Q_t \in \R^{k^* \times k_{t-1}}$. Therefore
\begin{align}
    B_t = U_t B_{t-1} = P_t \Lambda_t Q_t Q_{t-1}.
    \label{eq:new_SVD}
\end{align} Note that $Q=Q_t Q_{t-1} \in \R^{k^* \times \ds}$ satisfies $Q Q^\top=I$, since both $Q_t$ and $Q_{t-1}$ satisfies this. Therefore ~\eqref{eq:new_SVD} forms an SVD decomposition of $B_t$, and due to the uniqueness of non-zero singular values up to permutation, we have $k^* = k_t$.

Therefore $Q_t \in \R^{k_t \times k_{t-1}}$ satisfies 
$\forall e \in \cE_t,  Q_t Q_{t-1} \Sigma_2^e Q_{t-1}^\top Q^\top = C_t''$.

Applying Corollary~\ref{cor:recursive_informal} with $P=Q_{t-1}$, $Q=Q_t$, with probability $1-O(\exp{(-\ds)})$, no $Q_t$ satisfies $\forall e \in [E]$, $Q_t Q_{t-1} \Sigma_2^e Q_{t-1}^\top Q_t^\top=C_t$ for some constant  $C_t \in \R^{r_t \times r_t}$.

For the last iteration $t=T$, $E_{T}=3$. We assume without loss of generality $rank(B_{T-1})=k_{T-1} \in \{1, 2\}$, since we can always half the spurious dimensions $r_t-r \le (r_{t-1}-r)/2$ until $r_{t-1}-r=2$.

Lemma~\ref{lem:upper_2to1} and Lemma~\ref{lem:upper_1to0} in Appendix~\ref{sec:upper_proof} deal with the cases when $k_{T-1}=2$ and $k_{T-1}=1$, respectively. Suppose $rank(B_{T-1})=2$, its associated orthonormal matrix $Q_{T-1} \in \R^{2 \times \ds}$. Lemma~\ref{lem:upper_2to1} says that with yields that, with probability 1, no vector on the unit circle $q_T \in \mathbb{S}^1$ satisfies $q_T^\top Q_{T-1} (\Sigma_2^e-\Sigma_2^{e+1}) Q_{T-1}^\top q_T=0$ for $e \in \{1, 2\}$. Suppose $rank(B_{T-1})=1$, its associated unit-norm vector $q_{T-1}  \in \R^{\ds}$. Lemma~\ref{lem:upper_1to0} says that with probability 1, no non-zero scalar $q_T$ satisfies $q_T^2 q_{T-1}^\top (\Sigma_2^1-\Sigma_2^{2}) q_{T-1} =0$. Combining the two cases, with probability $1-O((T-1) \exp(-\ds))$, $rank(B_{T})=k_{T}=0$.
\end{proof}
\section{Experiments}
\label{sec:experiments}

\begin{figure}[ht]
\centering
\begin{minipage}[b]{0.49\linewidth}
\includegraphics[width=\linewidth]{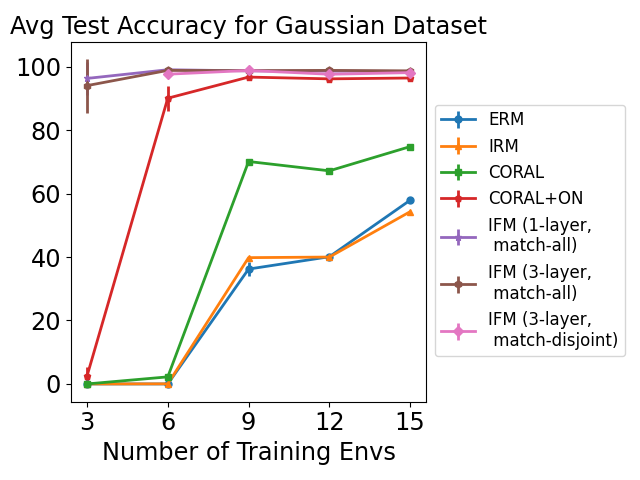}
  \caption{For Gaussian dataset, our algorithm IFM achieves highest test accuracy with the same number of training environments.}
\label{fig:1}
\end{minipage}
\hfill
\begin{minipage}[b]{0.49\linewidth}
\includegraphics[width=\linewidth]{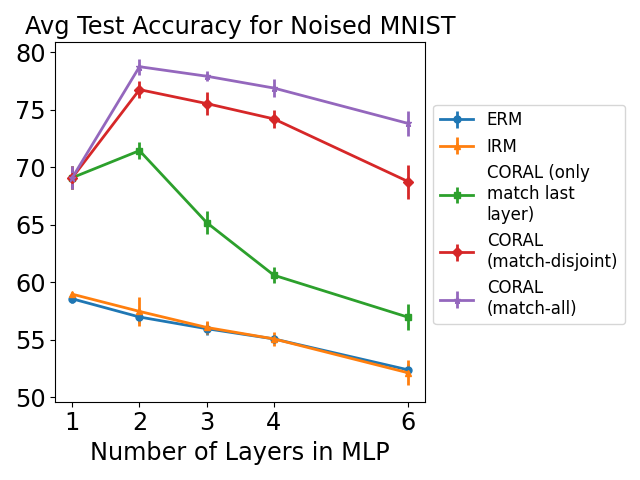}
  \caption{For Noised MNIST, matching feature distributions from multiple layers improves over naive CORAL across different architectures.}
\label{fig:2}
\end{minipage}
\end{figure}

In light of the differences between IFM and CORAL discussed in section~\ref{subsec:alg}, we test several questions inspired by our theory: (Q1) Do feature matching algorithms (IFM and CORAL) have much smaller environment complexity compared to ERM and IRM, with finite samples drawn from data models similar to our assumptions? (Q2) Can decoupling feature matching and supervised training of the classifier (IFM) improve over joint training (CORAL)? (Q3) For neural network featurizers, can matching feature distributions at multiple layers improve over matching at only the last layer (naive CORAL)? (Q4) Can matching disjoint sets of environments at each layer perform as well as matching all environments at all layers? (Q5) Is it important to shrink feature dimensions? We use two tasks to investigate those questions empirically. Appendix~\ref{sec:experiments_app} contains additional details.

\paragraph{Gaussian dataset} is a binary classification task that closely reflects our assumptions in section~\ref{sec:problem-setup}. We take $r=3$, $\ds=32$, $\mu_1 = \mathbf{1}_{r}$, $\Sigma_1 = I_{r}$, $\mu_2^e \sim \cN(0, 10 I_{\ds})$, and $\Sigma_2^e = G_e G_e^\top$. We use $1k$ samples per environment and vary the number of training / test environments from $E=3$ to $E=15$.

\paragraph{Noised MNIST} is a 10-way semi-synthetic classification task modified from~\citet{mnist} to test generalization of our theory to multi-class classification and different neural network architectures. The construction is inspired by the situation where certain background features spuriously correlate with labels (``most cows appear in grass and most camels appear in sand")~\citep{Beery_2018_ECCV, arjovsky2019invariant, aubin2021linear}, but the covariance of the background features changes across environments. Concretely, we divide the 60k images into $E=12$ groups. Each group is further divided into a training and a test environment with ratio 9:1. We add an additional row of noise ($28$ pixels) to the original grayscale digits of dimension $28 \times 28$. In training environments, the added noise is the spurious feature that, conditioned on the label, has identical mean but changing covariances across environments. In test environments, the noise is uncorrelated with the label.

\paragraph{Algorithms and architectures.}
For Gaussian dataset we use linear predictors. \textbf{IRM} follows the implementation in~\citet{arjovsky2019invariant}; \textbf{CORAL} jointly minimizes average supervised loss on training environments and $L_{coral}$, which is the average of squared distances in conditonal feature means (in $l_2$ norm) and covariances (in Frobenius norm) between adjacent training environments; \textbf{CORAL+ON} adds orthonormal penalty loss $L_{on}(U)=\|U U^\top-I\|_F^2$ where $U$ is the featurizer; \textbf{IFM} is our Algorithm~\ref{alg}, where for each layer $U_t$, the training objective is $L_t(U_t) = \lambda_1 L_{coral} + \lambda_2 L_{on}$. We test IFM with 1 vs. 3-layer featurizers, either matching all (\textbf{match-all}) or a disjoint set of training environments  (\textbf{match-disjoint}) at each layer.

For Noised MNIST we use ReLU networks with 1 to 6 layers.
Here the unsupervised feature matching stage of IFM would fail to extract features informative of the label; nonetheless, our theory inspires us to test whether bridging the gap between IFM and CORAL can improve test accuracy. Thus, we compare naive \textbf{CORAL} which only matches feature distributions at the last layer, to variants that match at all layers post-activation, using either all (\textbf{match-all}) or a disjoint subset of training environments (\textbf{match-disjoint}) per layer.

\paragraph{Results.}
(Q1) Figures~\ref{fig:1} and~\ref{fig:2} show that IFM and CORAL have much smaller environment complexity compared to ERM and IRM in both datasets. (Q2) In Gaussian dataset, IFM improves over CORAL. (Q3) In Noised MNIST dataset, matching feature distributions at multiple layers (CORAL match-all, CORAL match-disjoint) improves over matching at only the last layer (CORAL). (Q4) In both datasets, matching disjoint sets of environments at each layer (IFM match-disjoint, CORAL match-disjoint) is almost as good as matching all environments at all layers (IFM match-all, CORAL match-all) while saving computation. (Q5) In Noised MNIST dataset (Table~\ref{tab:shrink} in Appendix~\ref{sec:experiments_app}), shrinking feature dimensions is crucial for the advantage of feature matching at multiple layers, e.g. matching features at 3 layers with widths $[24, 24, 24]$ does not significantly improve over matching features at the last layer (CORAL). Overall, our results suggest that practitioners may benefit from feature matching algorithms when the data is similar to our assumed model, and may get additional advantage via matching at multiple layers with diminishing dimensions, echoing existing empirical works~\citep{pmlr-v37-long15, NIPS2017_a8baa565}.

\section{Conclusion}

This work presents the first domain generalization algorithm which provably recovers an invariant predictor with a number of environments that scales sub-linearly with the spurious feature dimension. Our results demonstrate that generalization which does not suffer from the ``curse of dimensionality'' is possible, and based on our theory we believe the use of an \emph{iterative} approach is a key insight which could lead to additional positive results for out-of-distribution generalization. Notably, this work also represents the first theoretical justification for the empirical success of existing algorithms which use feature distribution matching. However, there remains much room for improvement. Our results are for a linear data model with fairly special assumptions. It would be interesting to understand data models in which the observables are a non-linear function of the latent variables.
\section*{Acknowledgements}
We would like to thank Colin Wei, Sang Michael Xie, Ananya Kumar, and Ciprian Manolescu for helpful discussions. YC is supported by Stanford Graduate Fellowship and NSF IIS 2045685. MS is supported by NSF and Stanford Graduate Fellowships. TM acknowledges support of Google Faculty Award and NSF IIS 2045685.

\bibliographystyle{plainnat}
\bibliography{references}

\newpage
\onecolumn
\appendix
\section{Appendix}

\subsection{Proof of Lemma~\ref{thm:upper_ours}}
\label{sec:upper_proof}
The first lemma says IFM always finds $U_t$ that uses all invariant dimensions.
\begin{lemma}
\label{lem:preserve}
Let $A_0 \in \R^{d \times r}, B_0 \in \R^{d \times \ds}$ be the left $r$ and right $\ds$ columns of $S$. Define projection matrix onto the column span of $B$, $P_{B_0}=B_0 (B_0^\top B_0)^{-1} B_0^\top$. Suppose orthonormal $U_t \in \R^{r_t \times d}$ satisfies that $U_t S = [A_t, B_t]$ where $rank(U_t (I-P_{B_0}) A_0) < r$, and for all $e \in \cE$, $U S \Sigma^e S^\top U = C_t \in \R^{r_t \times r_t}$. Then there exist orthonormal $U'_t \in \R^{r_t' \times d}$ such that $r_t' > r_t$, $rank(U_t' (I-P_{B_0}) A_0) > rank(U_t (I-P_{B_0}) A_0)$, and for all $e \in \cE$, $U_t' S \Sigma^e S^\top {U_t'}^\top = C'_t \in \R^{r_t' \times r_t'}$.
\end{lemma}
\begin{proof}
We construct $U_t'$ by adding one additional row $u^+$ to $U_t$. Denote the columns of $(I-P_{B_0}) A_0$ as $a_1, \dots, a_r \in \R^d$. Since $U_t (I-P_{B_0}) A_0$ does not have full column rank, there is one column that can be written as linear combination of others. Assume without loss of generality that $U_t a_0 = \sum_{j=1}^r \alpha_j U_t a_j$, which implies that $U_t (a_0-\sum_{j=1}^r \alpha_j  a_j)=0$. Since $(I-P_{B_0}) A_0$ has full column rank $r$, $a^+ := a_0-\sum_{j=1}^r \alpha_j  a_j \ne 0$. Define $u^+ := a^+ / \|a^+\|_2$. Since $U_t a^+ = 0$, we have that $u_i^\top u^+ = 0$, for all existing rows of $U_t$ ($i \in [r_t]$). Furthermore, since each $a_j$ is orthogonal to the column space of $B_0$, ${u^+}^\top B_0 = 0$. Hence $U'_t = \begin{bmatrix}U_t \\ u^+\end{bmatrix}$ is orthonormal, $r_t' = r_t+1$, and $U'_t B_0 = \begin{bmatrix}U_t B_0 \\ \mathbf{0}_{1 \times \ds}\end{bmatrix}$ so
\begin{align*}
    U'_t S \Sigma^e S^\top {U'_t}^\top = U'_t A_0 \Sigma_1 A_0^\top {U'_t}^\top + U'_t B_0 \Sigma_2^e B_0^\top {U'_t}^\top = U'_t A_0 \Sigma_1 A_0^\top {U'_t}^\top + \begin{bmatrix}
    U_t B_0\Sigma_2^e B_0^\top U_t^\top & \mathbf{0}_{\ds \times 1} \\
    \mathbf{0}_{1 \times \ds} & 0
    \end{bmatrix}
\end{align*} which is constant for all $e \in \cE$.
\end{proof}

We use the following lemma to discretize the space of orthonormal matrices $\cQ =\{Q: Q Q^\top = I_k, Q \in \R^{k \times {\ds}}\}$. For any $Q, Q'\in \cQ$, we define the metric $\rho(Q, Q') = \|Q^\top Q - {Q'}^\top Q'\|_F$. We recall the following lemma about the existence of a cover of $\cQ$ with respect to the metric $\rho$: 
\begin{lemma}[Proposition 8 of~\citet{pajor1998metric}]
\label{lem:covering}
For $1 \le k \le \ds/2$, there exists absolute constant $c_3$ and covering $\Tilde{Q} \subset \cQ$ such that for all $\epsilon>0$, $|\Tilde{Q}| \le (c_3 \sqrt{k}/\epsilon)^{k (\ds-k)}$, and $\forall Q^* \in \cQ$, $\exists Q \in \Tilde{Q}$ such that $\rho(Q, Q^*) \le \epsilon$.
\end{lemma}
For any odd integer $e < E$, define $\Delta_2^e = \Sigma_2^e - \Sigma_2^{e+1}  = (\overline{\Sigma_2^e}-\overline{\Sigma_2^{e+1}}) + 
(G_e G_e^\top - G_{e+1} G_{e+1}^\top)$. 

For any $Q \in \cQ$, let $q_i$ be the $i$-th row of $Q$, for $i \in [k]$.
Let $Z_{i j e}=(q_i^\top \Delta_2^e q_j)^2$. Define $A_{i j e} = q_i^\top (\overline{\Sigma_2^e} -\overline{\Sigma_2^{e+1}}) q_j$, and $A = \sum_{\text{odd }e<E, i, j \in [k],i \ne j} A_{i j e}^2$. The main lemma below shows that the sum of $Z_{i j e}$'s are bounded away from 0.

\begin{lemma}
\label{lem:main}
There exists constants $c_1, c_2, b_1, b_2 > 0$ such that for any integer $2 \le k \le \ds/2$, for all  $E$ satisfying
\begin{align*}
    b_1 \frac{\ds-k}{k-1} \max\left\{1, \log\left(\frac{D}{ (k-1) \ds}\right), \log\left(\frac{\ds}{ k-1}\right)\right\} < E <  b_2 \ds,
\end{align*} where $ \max_e{\|\overline{\Sigma_2^e}\|_2^2} \le D$ for some constant $D$, with probability $1-c_1 \exp{(-\ds)}$, for all $Q \in \cQ$, 
\begin{align*}
   \sum_{\text{odd }e<E, i, j \in [k],i \ne j} Z_{i j e} > c_2 (A + E k (k-1) \ds).
\end{align*}
\end{lemma}
\begin{proof}
For any odd $e<E$ and $i \in [k]$, by definition
\begin{align*}
    \sum_{j \ne i} Z_{i j e} = \sum_{j \ne i} (A_{i j e} + q_i^\top G_e G_e^\top q_j - q_i^\top G_{e+1} G_{e+1}^\top q_j)^2
\end{align*}
Define $V_{i,e}=G_e q_i$ for $e \in [E], i \in [k]$. For fixed orthonormal $Q$, $V_{i,e} \sim \cN(0, I_{\ds})$ and the ensemble $\{V_{i,e}\}_{i \in [k], e \in [E]}$'s is independent. Therefore
\begin{align*}
    q_i^\top G_e G_e^\top q_j - q_i^\top G_{e+1} G_{e+1}^\top q_j =  V_{i,e}^\top V_{j,e} -V_{i,e+1}^\top V_{j,e+1}
\end{align*}
For further simplification, we define $W_{i,e} = [V_{i,e}; V_{i, e+1}] \in \R^{2\ds}$, and $I^* = [I_{\ds}, \mathbf{0}; \mathbf{0}, -I_{\ds}]$, so
\begin{align*}
    V_{i,e}^\top V_{j,e} -V_{i,e+1}^\top V_{j,e+1} = W_{i,e}^\top I^* W_{j,e}
\end{align*}
We use the following lemma to decouple the correlations between $W_{i,e}^\top I^* W_{j,e}$ and $W_{i,e}^\top I^* W_{j',e}$ for $j' \ne j, j' \ne i, i \ne j$:
\begin{lemma}[Theorem 1 of~\citet{de1995decoupling}]
\label{lem:PMS95}
Suppose $\{X_i\}$ ($i \in [k]$) are independent random variables, $X_i$ and $Y_i$ have the same distribution. There exists some absolute constant $c_4$ such that
\begin{align*}
    \Pr\left[\left\lvert\sum_{i,j \in [k], i \ne j} f(X_i,X_j)\right\rvert \ge t\right] \le c_4 \Pr\left[\left\lvert\sum_{i,j \in [k], i \ne j} f(X_i,Y_j)\right\rvert \ge t/c_4\right].
\end{align*}
\end{lemma}
We apply Lemma~\ref{lem:PMS95} with $X_i = W_{i,e}$ and $f(X_i,X_j) = Z_{i j e} - \E[Z_{i j e}]$ to get
\begin{align*}
    \Pr\left[\left\lvert\sum_{i,j \in [k], i \ne j} Z_{i j e} - \E[Z_{i j e}]\right\rvert \ge t\right] \le c_4 \Pr\left[\left\lvert\sum_{i,j \in [k], i \ne j} Z'_{i j e} - \E[Z'_{i j e}]\right\rvert \ge t/c_4\right].
\end{align*} where $Y_{i, e}$ and $X_{i, e}$ are identically distributed and
\begin{align*}
    Z'_{i j e} = A_{i j e}^2+ 2 A_{i j e} X_{i,e}^\top I^* Y_{j,e} + (X_{i,e}^\top I^* Y_{j,e})^2.
\end{align*} Note that $\{Z'_{i j e}\}$ and $\{Z''_{i j e}\}$ are identically distributed, where
\begin{align*}
    Z_{i j e}''=A_{i j e}^2+ 2 A_{i j e} X_{i,e}^\top Y_{j,e} + (X_{i,e}^\top Y_{j,e})^2.
\end{align*}

Below we first consider the randomness in $\{Y_{i,e}\}$, and prove that with high probability $\{Y_{i,e}\}$ satisfies some good properties; we then show the concentration of $\sum_{i, j, e}Z_{i j e}''$ conditioned on the event that  $\{Y_{i,e}\}$ satisfies these properties.

First, for fixed $Q$, since $Y_{i,e}=[G_e v_i; G_{e+1} v_i] \sim \cN(0, I_{2\ds})$, if we write
$Y_e = [Y_{1, e}; \dots;Y_{k, e}] \in \R^{k \times 2 \ds}$, it is a random matrix with iid standard normal entries. We show that the $\|Y_e\|_F^2 = \Theta(k \ds)$ with high probability.
The following lemma is a standard concentration bound for chi-squared variable:
\begin{lemma}[Corollary of Lemma 1 in~\citet{laurent2000}]
\label{lem:chi_square}
Suppose $Z_i \sim \cN(0,1)$ for $i \in [n]$. For any $t>0$,
\begin{align*}
    \Pr\left[ \sum_{i=1}^n Z_i^2 \ge n+2\sqrt{n t}+2t\right] &\le \exp(-t), \\
    \Pr\left[\sum_{i=1}^n Z_i^2 \le n-2\sqrt{n t}\right] &\le \exp(-t).
\end{align*}
\end{lemma}
Applying Lemma~\ref{lem:chi_square} to $n= E k  \ds$ entries of $\{Y_e\}_{e=1}^E$ and setting $t= E k \ds/16$ we get with probability $1-2 \exp(-E k \ds/16)$,
\begin{align}
    \frac{E k \ds}{2} \le \sum_e \|Y_e\|_F^2 \le \frac{13 E k \ds}{8}.
    \label{eq:y_e_fro_norm}
\end{align}

Second, we show that with high probability over the randomness of $G_e$, $\|Y_e\|_2$ viewed as a function of $Q$ satisfies $\|Y_e\|_2 =O(\sqrt{\ds})$ for all orthonormal $Q$. We use the following lemma to upper bound $\|G_e\|_2$:
\begin{lemma}[Corollary 5.35 of~\citet{Vershynin12}]
\label{lem:vershynin}
Suppose $G \in \R^{D \times d}$ and $[G]_{i j} \sim \cN(0, 1)$ for all $i \in [D],j \in [d]$. For every $t \ge 0$, with probability $1-2\exp(-t^2/2)$, 
\begin{align*}
    \|G\|_2 \le \sqrt{D}+\sqrt{d}+t
\end{align*}
\end{lemma}
Applying Lemma~\ref{lem:vershynin} with $G = [G_e; G_{e+1}]$, $D=2 \ds$, $d=\ds$, $t=\sqrt{\ds}$, we get with probability $1-2 \exp(-\ds/2)$, $ \|G\|_2 \le (2+\sqrt{2}) \sqrt{\ds} $, and therefore for all orthonormal $Q \in \R^{k \times \ds}$,
\begin{align}
    \|Y_e\|_2  = \|Q G^\top\|_2 \le \|Q\|_2 \|G\|_2\le (2+\sqrt{2}) \sqrt{\ds}.
    \label{eq:y_e_2_norm}
\end{align}

For any odd $e<E$, $i \in [k]$, and fixed $Y_e$, we prove $P_{e i}=\sum_{j \ne i}Z''_{i j e}$ concentrates. Once we fix $Y_e$, the $Er/2$ random variables $\{P_{e i}\}$ are independent, so the concentration of their sum is immediate. Let $Y_{-i, e}$ be $Y_e$ without the $i$-th row,
\begin{align}
    P_{e i}=\sum_{j \ne i}Z''_{i j e} = \sum_{j \ne i} A_{i j e}^2 + 2 X_{i,e}^\top \left( \sum_{j \ne i} A_{i j e} Y_{j,e}\right) + X_{i,e} ^\top Y_{-i,e} Y_{-i,e}^\top X_{i,e}
    \label{eq:pei}
\end{align}
Define $B_{i,e} = Y_{-i,e} Y_{-i,e}^\top$. Let $a_{i,e} \in \R^{k-1}$ be the column vector consisting of $A_{i j e}$ for $j \ne i$.

Since $X_{i,e} \sim \cN(0, I_{2\ds})$, $X_{i,e}^\top \left( \sum_{j \ne i} A_{i j e} Y_{j,e}\right)$ is a Gaussian variable with mean 0 and variance $a_{i,e}^\top B_{i,e} a_{i,e} \le \|a_{i,e}\|_2^2 \|B_{i,e}\|_2$, so by Hoeffding's inequality, for all $t \ge 0$,
\begin{align}
    \Pr\left[2X_{i,e}^\top \left( \sum_{j \ne i} A_{i j e} Y_{j,e}\right) > t \mid Y_e\right] \le \exp{\left(-\frac{ t^2}{8\|a_{i,e}\|^2 \|B_{i,e}\|_2}\right)}.
    \label{eq:middle}
\end{align}
By Hanson-Wright Inequality (e.g. Theorem 1.1 of~\citet{10.1214/ECP.v18-2865}), there exists constant $c_5$ such that
\begin{align}
    \Pr\left[\E[X_{i,e}^\top B_{i,e} X_{i,e}]-X_{i,e}^\top B_{i,e} X_{i,e} >t \mid Y_e\right] \le  \exp\left(-c_5 \min\left\{\frac{t^2}{\|B_{i,e}\|_F^2}, \frac{t}{\|B_{i,e}\|_2}\right\}\right).
    \label{eq:hanson}
\end{align}

Combining equations~\eqref{eq:pei}, \eqref{eq:middle}, \eqref{eq:hanson}, we get
\begin{align*}
    \Pr\left[\E[P_{e i}] -P_{e i} > t \mid Y_e\right] \le \exp{\left(-\frac{ t^2}{32\|a_{i,e}\|_2^2 \|B_{i,e}\|_2}\right)}+ \exp\left(-c_5 \min\{\frac{t^2}{4\|B_{i,e}\|_F^2}, \frac{t}{2\|B_{i,e}\|_2}\}\right).
\end{align*}
Summing over all $e \in [E]$ and $i \in [k]$ we get
\begin{align*}
    \Pr\left[\E\left[\sum_{e, i}P_{e i}\right] -\sum_{e, i}P_{e i}  > t \mid Y_1, \dots, Y_E\right] \le \exp{\left(-\frac{ t^2}{32\sum_{e, i}\|a_{i,e}\|_2^2 \|B_{i,e}\|_2}\right)} \\+ \exp\left(-c_5 \min\left\{\frac{t^2}{4\sum_{e, i} \|B_{i,e}\|_F^2}, \frac{t}{2\max_{e, i}\|B_{i,e}\|_2}\right\}\right).
\end{align*}
Note that $\E[X_{i,e}^\top B X_{i,e}]=\E\sum_{j \ne i} (X_{i,e}^\top Y_{j,e})^2 = \|Y_{-i,e}\|_F^2$ so
\begin{align*}
    E\left[\sum_{e, i} P_{ei} \right] = \sum_{e, i}\|a_{i,e}\|_2^2+ \sum_{e, i} \|Y_{-i,e}\|_F^2 = A + (k-1) \sum_{e} \|Y_e\|_F^2.
\end{align*}
Since $\|B_{i,e}\|_2 \le \|Y_e\|_2^2$, $\|B_{i,e}\|_F^2 \le \|Y_{-i,e}\|_F^2 \|Y_e\|_2^2$, taking $t = \frac{1}{2} E[\sum_{e, i} P_{ei} ]$,
\begin{align*}
    \Pr\left[\sum_{e, i}P_{e i} < \frac{1}{2}\left(A + (k-1) \sum_{e} \|Y_e\|_F^2\right) \mid Y_1, \dots, Y_E\right] \le \exp{\left(-\frac{ (A + (k-1) \sum_{e} \|Y_e\|_F^2)^2}{128\sum_{e, i}\|a_{i,e}\|_2^2 \|Y_{e}\|_2^2}\right)} \\+ \exp\left(-c_5 \min\left\{\frac{(k-1)^2 (\sum_{e} \|Y_e\|_F^2)^2}{16 (k-1) \sum_{e} \|Y_{e}\|_F^2 \|Y_e\|_2^2}, \frac{(k-1) \sum_{e} \|Y_e\|_F^2}{2\max_{e}\|Y_{e}\|_2^2}\right\}\right).
\end{align*}
Let $\cE_1$ denote the event that for all odd $e < E$, $[G_e; G_{e+1}] \in \R^{2\ds \times \ds}$ denote the matrix with $G_e, G_{e+1} \in \R^{\ds \times \ds}$ in its first and last $\ds$ rows, respectively, we have
\begin{align*}
    \|[G_e; G_{e+1}]\|_2 \le (2+\sqrt{2})\sqrt{\ds}.
\end{align*} Due to equation~\eqref{eq:y_e_2_norm} and the union bound, $\Pr[\cE_1] \ge 1- E\exp(-\ds/2)$.
Conditioned on $\cE_1$, for all $Q \in \cQ$ and odd $e<E$,
\begin{align*}
    \|Y_e\|_2  \le (2+\sqrt{2}) \sqrt{\ds}.
\end{align*}

Let $\cE_2$ denote the event that for all cover elements $Q \in \Tilde{Q}$,
\begin{align*}
    \frac{E k \ds}{2} \le \sum_e \|Y_e\|_F^2 \le \frac{13 E k \ds}{8}.
\end{align*}
Due to equation~\eqref{eq:y_e_fro_norm} and the union bound, $\Pr\left[ \cE_2 \right] \ge 1-2 |\Tilde{Q}|  \exp{\left(-E k \ds/16\right)}$.

Conditioned on $\cE_1$ and $\cE_2$, for fixed $Q \in \Tilde{Q}$, there exists constants $c_6, c_7$ such that 
\begin{align*}
    \Pr\left[\sum_{e, i}P_{e i} < \frac{1}{2}A + \frac{1}{4}E k(k-1) \ds \right] \le \exp{\left(-c_6 \frac{ (A + E k(k-1) \ds)^2}{A \ds}\right)} \\+ \exp\left(-c_7 \min\left\{\frac{(k-1)^2 E^2 k^2 \ds^2}{E k(k-1) \ds^2}, \frac{E k (k-1) \ds}{\ds}\right\}\right),
\end{align*} which implies there exists constants $c_8$ such that
\begin{align*}
    \Pr\left[\sum_{e, i}P_{e i} < \frac{1}{2}A + \frac{1}{4}E k(k-1) \ds \right] \le \exp\left(-c_8 \min\left\{\frac{ (A + E k(k-1) \ds)^2}{A \ds}, E k(k-1)\right\}\right).
\end{align*}

Note that we always have $\frac{ (A + E k(k-1) \ds)^2}{A \ds} \ge E k (k-1)$. To see this, for $A> E k (k-1) \ds$, $\frac{(A + E k(k-1) \ds)^2}{A \ds} > \frac{A}{\ds} > E k (k-1)$. For $A \le E k (k-1) \ds$,  $\frac{(A + E k(k-1) \ds)^2}{A \ds} \ge \frac{(E k (k-1) \ds)^2}{E k (k-1) \ds^2} = E k (k-1)$.

In other words, with probability $1-\delta$, where
\begin{align*}
    \delta = E\exp(-\ds/2)+2 |\Tilde{Q}|  \exp{\left(-E k \ds/16\right)} +|\Tilde{Q}| \exp{\left(-c_8 E k (k-1) \right)},
\end{align*} all $Q \in \Tilde{Q}$ satisfies $\sum_{e, i}P_{e i} \ge \frac{1}{4}(A + E k(k-1) \ds)$.
 Combined with Lemma~\ref{lem:PMS95}, with probability $1- c_9 \delta$, all  $Q \in \Tilde{Q}$ satisfies
$\sum_{e, i, j}Z_{i j e} < c_{10} (A + E k(k-1) \ds)$ for some constants $c_9, c_{10}$.

For any $Q^* \in \cQ$, let $Q$ be the element in the cover closest to it, so that $\rho(Q, Q^*) = \|Q^\top Q - {Q^*}^\top {Q^*}\|_F \le \epsilon$. Let $q_i^*$ be the $i$-th row of $Q^*$, and $Z^*_{i j e} = ({q_i^*}^\top \Delta_2^e q_j^*)$. Then
\begin{align*}
    \sum_{e i j} Z^*_{i j e} &= \sum_{e} \|Q^* \Delta_2^e {Q^*}^\top\|_F^2 \\
    &= \sum_{e} \|\Delta_2^e {Q^*}^\top Q^*\|_F^2 \\
    &\ge \frac{1}{2}\sum_e \|\Delta_2^e {Q}^\top Q\|_F^2 - \|\Delta_2^e \left(Q^\top Q - {Q^*}^\top {Q^*}\right) \|_F^2 \\
    &\ge \frac{1}{2}\sum_{e i j} Z_{i j e} - \|\Delta_2^e\|_2^2 \rho(Q, Q^*)^2.
\end{align*}
Since $\|\Delta_2^e\|_2^2 \le 2\|\overline{\Sigma_2^e}\|_2^2 + 2\|G_e G_e^\top\|_2^2$, and conditioned on $\cE_1$, $\|G_e G_e^\top\|_2^2 \le c_{11} \ds^2$ for all $e$, if $ \max_e{\|\overline{\Sigma_2^e}\|_2^2} \le D$ for some constant $D$, we have with probability $1-\delta$,
\begin{align}
     \sum_{e i j} Z^*_{i j e} \ge \frac{c_{10}}{2} (A+E k (k-1) \ds) - 2 E(D +c_{11} \ds^2) \epsilon^2.
     \label{eq:any_q}
\end{align}
We choose $\epsilon^2 < \frac{c_{10} k (k-1) \ds}{8 (D+ c_{11} \ds^2)}$  so that $2E(D+c_{11} \ds^2) \epsilon^2 < \frac{c_{10}}{4} E k (k-1) \ds$.

With this choice of $\epsilon$, by equation~\eqref{eq:any_q} we have
\begin{align*}
     \sum_{e i j} Z^*_{i j e} \ge \frac{c_{10}}{4} (A+E k (k-1) \ds).
\end{align*}
By Lemma~\ref{lem:covering}, $\log(|\Tilde{Q}|) \le k(\ds-k) \log(c_3\sqrt{k}/\epsilon) \le  c_{12} k (\ds-k) \log{\left(\frac{D}{(k-1) \ds}+\frac{\ds}{k-1}\right)}$.

Therefore there exists $b_1, b_2>0$ such that for $E$ satisfying
\begin{align*}
    b_2 \ds > E > b_1 \frac{\ds-k}{k-1} \max\left\{1, \log\left(\frac{D}{ (k-1) \ds}\right), \log\left(\frac{\ds}{ k-1}\right)\right\},
\end{align*} we have
\begin{align*}
    \delta &\le \exp(-\ds/2+\log{(b_2 \ds)}) + 2  \exp{\left(c_{12} k (\ds-k) \log{\left(\frac{D}{(k-1) \ds}+\frac{\ds}{k-1}\right)}-E k \ds/16\right)} \\ &+  \exp\left(c_{12} k (\ds-k) \log{\left(\frac{D}{(k-1) \ds}+\frac{\ds}{k-1}\right)} -c_8 E k (k-1) \right) \\
    &\le c_1 \exp{(-\ds)}
\end{align*} for some constant $c_1$.
Therefore with probability $1-c_1 \exp{(-\ds)}$, for all $Q^* \in \cQ$, and $c_2 = c_{10}/4$,
\begin{align*}
     \sum_{e i j} Z^*_{i j e} \ge c_2 (A+  E k (k-1) \ds).
\end{align*}
\end{proof}

\begin{corollary}[Corollary of Lemma~\ref{lem:main}]
\label{cor:recursive}
Suppose $2 \le k \le r/2 \le \ds/2$. Let $\cP = \{P\in \R^{r \times \ds}: P P^\top =I_r\}$, $\cQ = \{Q \in \R^{k \times r}: Q Q^\top =I_k\}$. For fixed $P \in \cP$, there exists constants $c_1, c_2, b_1, b_2 > 0$ such that for all  $E$ satisfying
\begin{align*}
    b_1 \frac{r-k}{k-1} \max\left\{1, \log\left(\frac{D}{ (k-1) \ds}\right), \log\left(\frac{\ds}{ k-1}\right)\right\} < E <  b_2 \ds,
\end{align*} where $ \max_e{\|\overline{\Sigma_2^e}\|_2^2} \le D$ for some constant $D$, with probability $1-c_1 \exp{(-\ds)}$, for all $Q \in \cQ$, 
\begin{align*}
   \sum_{\text{odd }e<E} \|Q P \Delta_2^e P^\top Q^\top\|_F^2 > c_2  E k (k-1) \ds.
\end{align*}
\end{corollary}

\begin{proof}
The proof mostly follows that of Lemma~\ref{lem:main}, with a few modifications below. We discretize over $\cQ$ and get a $\epsilon$-covering $\Tilde{Q}$ of size $(c_3\sqrt{k}/\epsilon)^{r(r-k)}$.

For any $Q \in \cQ$, let $v_i$ be the $i$-th row of $Q P$ and define $Z_{i j e}, A_{i j e}$ accordingly. For any $Q^* \in \cQ$, let $Q$ be its cover element, so $\rho(Q, Q^*) = \|Q^\top Q - {Q^*}^\top {Q^*}\|_F \le \epsilon$. Let $q_i^*$ be the $i$-th row of $Q^* P$, and $Z^*_{i j e} = ({q_i^*}^\top \Delta_2^e q_j^*)$. Then
\begin{align*}
    \sum_{e i j} Z^*_{i j e} &= \sum_{e} \|Q^* P \Delta_2^e P^\top {Q^*}^\top\|_F^2 \\
    &= \sum_{e} \|P \Delta_2^e P^\top {Q^*}^\top Q^*\|_F^2 \\
    &\ge \frac{1}{2}\sum_e \|P \Delta_2^e P^\top {Q}^\top Q\|_F^2 - \|P \Delta_2^e P^\top \left(Q^\top Q - {Q^*}^\top {Q^*}\right) \|_F^2 \\
    &\ge \frac{1}{2}\sum_{e i j} Z_{i j e} - \|P \Delta_2^e P^\top\|_2^2 \rho(Q, Q^*)^2 \\
    &\ge \frac{1}{2}\sum_{e i j} Z_{i j e} - \| \Delta_2^e \|_2^2 \rho(Q, Q^*)^2
\end{align*}
Thus with the same choice of $\epsilon$ as Lemma~\ref{lem:main}, $\log(|\Tilde{Q}|) \le k(r-k) \log(c_3\sqrt{k}/\epsilon) \le  c_{12} k (r-k) \log{\left(\frac{D}{(k-1) \ds}+\frac{\ds}{k-1}\right)}$. The rest of the argument is identical.
\end{proof}

\begin{lemma}
Let $\cP = \{P\in \R^{2 \times \ds}: P P^\top =I_2\}$. Suppose $\Sigma_2 = \overline{\Sigma_2^1}-\overline{\Sigma_2^2} + G_1 G_1^\top - G_2 G_2^\top$ and $\Sigma_2' = \overline{\Sigma_2^1}-\overline{\Sigma_2^3} + G_1 G_1^\top - G_3 G_3^\top$, where $G_e \in \R^{\ds \times \ds}$ and $[G_e]_{i j} \sim \cN(0,1)$ for all $e \in [3]$, $i, j \in [\ds]$. For fixed $P \in \cP$, with probability 1, no vector $q \in \R^2$ satisfies $\|q\|_2=1$ and
\begin{align*}
    q^\top \Sigma_2 q = 0, \quad q^\top \Sigma_2' q = 0.
\end{align*}
\label{lem:upper_2to1}
\end{lemma}
\begin{proof}
For any fixed $G_1, G_2$, consider the system of quadratic equations over two variables,
\begin{align*}
    \{q^\top \Sigma_2 q = 0, \|q\|_2=1\}.
\end{align*} With probability 1, it has at most 4 real solutions.
Conditioned on $G_1, G_2$, consider the third quadratic equation where the randomness is in $G_3$.
\begin{align*}
    \{q^\top \Sigma_2' q = 0\}.
\end{align*} With probability 1, any fixed solution from the first system does not satisfy this.
\end{proof}

The following lemma is trivial so proof is omitted:
\begin{lemma}
Suppose $p \in \R^{\ds}$ and $\|p\|_2=1$. Suppose $\Sigma_2 = \overline{\Sigma_2^1}-\overline{\Sigma_2^2} + G_1 G_1^\top - G_2 G_2^\top$, where $G_e \in \R^{\ds \times \ds}$ and $[G_e]_{i j} \sim \cN(0,1)$ for $e \in [2]$, $i, j \in [\ds]$. With probability 1, no scalar $q \ne 0$ satisfies
\begin{align*}
    q^2 p^\top \Sigma_2 p = 0.
\end{align*}
\label{lem:upper_1to0}
\end{lemma}

\subsection{Proof of Theorem~\ref{thm:ERM}}
\label{sec:ERM_proof}

\begin{proof}
Denote the unit-norm classifier $\beta$. For any environment with mean $(\mu_1, \mu_2^i)$ and covariance $\Sigma_1, \Sigma_2^i$, the accuracy of $\beta$ can be written
\begin{align*}
    \mathbb{E}[\mathbf{1}(\sign(\beta^\top x)=y] &= p(y=1)p(\beta^\top x \geq 0\mid y=1) + p(y=-1)p(\beta^\top x < 0\mid y=-1) \\
    &= \frac{1}{2}\left[1-\Phi\left(-\frac{\beta_1^\top \mu_1 + \beta_2^\top \mu_2^i}{\sqrt{\beta_1^\top \Sigma_1\beta_1 + \beta_2^\top \Sigma_2^i\beta_2}}\right)\right] + \frac{1}{2} \Phi\left(\frac{\beta_1^\top \mu_1 + \beta_2^\top \mu_2^i}{\sqrt{\beta_1^\top \Sigma_1\beta_1 + \beta_2^\top \Sigma_2^i\beta_2}}\right) \\
    &= \Phi\left(\frac{\beta_1^\top \mu_1 + \beta_2^\top \mu_2^i}{\sqrt{\beta_1^\top \Sigma_1\beta_1 + \beta_2^\top \Sigma_2^i\beta_2}}\right),
\end{align*}
where $\Phi$ is the standard normal CDF. Observe that $\Phi$ is monotone and that $\sigma_2^2 I \preceq \Sigma_2^i$. Therefore, a training accuracy of at least $\gamma$ on each environment implies that for each environment,
\begin{align*}
    \gamma &\leq {\Phi\left(\frac{\beta_1^\top \mu_1 + \beta_2^\top \mu_2^i}{\sqrt{\beta_1^\top \Sigma_1 \beta_1+\beta_2^\top \Sigma_2^i \beta_2}}\right)} \\
    &\leq {\Phi\left(\frac{\beta_1^\top \mu_1 + \beta_2^\top \mu_2^i}{\sqrt{\sigma_1^2 \|\beta_1\|^2+\sigma_2^2 \|\beta_2\|^2}}\right)}.
\end{align*}
For brevity, moving forward we will denote $\psi := \sqrt{\sigma_1^2 \|\beta_1\|^2+\sigma_2^2 \|\beta_2\|^2}$. Applying the inverse CDF (which is also monotone) and rearranging, we have
\begin{align*}
    \beta_2^\top\mu_2^i &\geq \psi\Phi^{-1}(\gamma) - \beta_1^\top \mu_1,
\end{align*}
which implies
\begin{align*}
    \beta_1^\top\mu_1 - \beta_2^\top\mu_2^i &\leq 2\beta_1^\top\mu_1 - \psi\Phi^{-1}(\gamma).
\end{align*}
If $\gamma \geq \Phi\left(\frac{2\|\mu_1\|}{\min(\sigma_1,\sigma_2)}\right) \geq \Phi\left(\frac{2\beta_1^\top\mu_1}{\psi}\right)$ then we have $\beta_1^\top\mu_1 - \beta_2^\top\mu_2^i \leq 0$ for all environments and therefore the classifier has accuracy $< \frac{1}{2}$ on all test environments.
\end{proof}

\subsection{Proof of Theorem~\ref{thm:irm}}

\begin{definition}

For a positive definite matrix $A\in Mat_{d\times d}(\mathbb R)$ and vector $b\in \mathbb R^d$, the associated ellipsoid $E_{A,b}\subseteq\mathbb R^d$ is given by

\[E_{A,b}=\{x\in\mathbb R^d:x^{\top}Ax-b^{\top}x=0\}.\]

\end{definition}

Observe that the origin is contained in any such ellipsoid $E_{A,b}$. Therefore, any collection of ellipsoids $E_{A_i,b_i}$ has the origin as a trivial point in its intersection. Our main result ensures the existence of another (non-trivial) intersection of any $d$ such ellipses whenever the vectors $b_i$ are linearly independent.

\begin{theorem}\label{thm:irm_main}

If $b_1,\dots,b_d\in\mathbb R^d$ are linearly independent and $A_1,\dots,A_d$ are positive-definite matrices, then 
\begin{align*}
    \left|\bigcap_{i=1}^d E_{A_i,b_i}\right|\geq 2.
\end{align*}

\end{theorem}

To prove this result we use technical tools from differential topology. The most central tool, Proposition~\ref{prop:mod2}, ensures that the total number of intersection points between two manifolds of complementary dimensions $k,d-k$ is even when certain generic tranversality conditions hold. Using these techniques, we show that $\left|\bigcap_{i=1}^d E_{A_i,b_i}\right|\geq 2$ for almost all matrices $A_1,\dots,A_d$, as long as $b_1,\dots,b_d$ are linearly independent. Then we use a continuity argument to extend the result to all positive definite matrices $A_1,\dots,A_d$. 

Throughout we say a function is \emph{smooth} to mean it is infinitely differentiable, i.e. $C^{\infty}$. All manifolds considered are smooth, i.e. they have a smooth structure. When $F(x,y)$ has two arguments we denote by $F_x$ the function $F_x(y)=F(x,y)$ of $y$ given by fixing $x$, and similarly define $F_y$. If $x\in X$ is a point in the smooth manifold $X$, we denote by $T_x(X)$ its \emph{tangent space}, which is intuitively the vector space of all tangent vectors to $X$ at $x$. The derivative of a smooth map $f:X\to Y$ at $x\in X$ is a linear map $df_x:T_x(X)\to T_{f(x)}(Y)$. 

\begin{definition}\cite[Chapter 1.5]{guilleminpollack}

Let $X,Y,Z$ be smooth manifolds (without boundary) such that $Z\subseteq Y$. The smooth map $f:X\to Y$ is \emph{tranverse} to $Z$ if for each $x\in X$ with $f(x)\in Z$, it holds that

\[\text{Image}(df_x)+T_{f(x)}(Z)=T_{f(x)}(Y).\]

If $X,Z\subseteq Y$ are both submanifolds of $Y$, we say they are transverse if the inclusion $\iota_X:X\hookrightarrow Y$ is transverse to $Z$. Equivalently, this means that for any $x\in X\cap Z$,

\[T_x(X)+T_x(Z)=T_x(Y).\]

\end{definition}

Roughly speaking, smooth two manifolds $X,Z$ are transversal if all intersection points are ``typical". For example, if $\dim(X)+\dim(Z)<\dim(Y)$, then $X,Z$ being transverse is equivalent to their intersection being empty. This corresponds to the intuition that their total dimension is too small for them to generically intersect. If $\dim(X)+\dim(Z)=\dim(Y)$, transversality rules out ``unstable" intersections such as a line tangent to a circle. 

\begin{proposition}\cite[Chapter 1.5]{guilleminpollack}\label{prop:transversality0}

The intersection $W=X\cap Z$ of two transversal submanifolds $X,Z\subseteq Y$ is itself a submanifold of $Y$, and $\dim(W)=\dim(X)+\dim(Z)-\dim(Y)$.

\end{proposition}

\begin{proposition}\cite[Chapter 2.3]{guilleminpollack}\label{prop:transversality1}

Suppose that $F:X\times S\to Y$ is a smooth map of manifolds, and let $Z$ be a sub-manifold of $Y$. If $F$ is transversal to $Z$, then for almost every $s\in S$, the map $f_s=F(\cdot,s):X\to Y$ is also transversal to $Z$.

\end{proposition}

\begin{proposition}\cite[Chapter 2.4, Exercise 5]{guilleminpollack}\label{prop:mod2}

Suppose the smooth, compact manifolds $X,Y\subseteq \mathbb R^d$ are transversal, and that $\dim(X)+\dim(Y)=d$. Then $|X\cap Y|$ is finite and even.

\end{proposition}

\begin{remark}

Proposition~\ref{prop:mod2} follows from the methods of \cite[Chapter 2.4]{guilleminpollack}, which shows that the parity of $|X\cap Y|$ is invariant under homotopy as long as transversality is enforced. One simply argues that by a homotopy $X\to X',Y\to Y'$, we can arrange that $|X'\cap Y'|=0$ by translating $X$ far away and invoking compactness.

\end{remark}

\begin{lemma}

The tangent space $T_0 E_{A,b}$ is exactly the orthogonal complement $b^{\perp}.$

\end{lemma}

\begin{proof}

Since $E_{A,b}$ is an ellipsoid, it is a smooth manifold of dimension $d-1$. If $\gamma:[0,1]\to E_{A,b}$ is a smooth curve with $\gamma(0)=0$, then we claim $\langle b,\gamma'(t)\rangle=0$. This suffices to prove the desired result since $\gamma'(t)$ can be any vector in $T_0 E_{A,b}$. Indeed, differentiating the equation for $E_{A,b}$ gives

\begin{align*}0&=2\frac{d}{dt}\langle 0,A\gamma(t)\rangle\\
&=\frac{d}{dt}\langle \gamma(t),A\gamma(t)\rangle|_{t=0}\\
&=\frac{d}{dt}\langle b,\gamma(t)\rangle|_{t=0} \\
&=\langle b,\gamma'(t)\rangle|_{t=0}.
\end{align*}
\end{proof}

Set $\mathcal A^{\circ}$ to be the set of all $d\times d$ strictly positive-definite matrices with distinct eigenvalues. Note that $\mathcal A^{\circ}$ is open in the space of all positive definite matrices, and its complement has Lebesgue measure $0$. Denote by $\mathbb S^{d-1}\subseteq\mathbb R^d$ the unit sphere so that $(c_1,\dots,c_d)\in\mathbb S^{d-1}$ if and only if $\sum_{i=1}^d c_i^2=1$.

\begin{proposition}\cite[Theorem 5.3]{serre}\label{prop:serre}

For any $A_0\in \mathcal A^{\circ}$, there is an open neighborhood $U_{A_0}\subseteq \mathcal A^{\circ}$ of $A_0$ such that the eigenvalues $\lambda_1(A)>\dots>\lambda_d(A)$ and associated orthonormal eigenvectors $v_1,\dots,v_d$ can be chosen to depend smoothly on the entries of $A\in U_{A_0}$.

\end{proposition}

We remark that is it impossible to make a \emph{globally} smooth choice of the eigenvectors and eigenvalues as above. This is because of problems caused by higher multiplicity eigenvalues, and also by the need to choose a sign for the eigenvectors.

\begin{lemma}\label{lem:ellipse}

For $A\in \mathcal A^{\circ}$ and non-zero $b\in\mathbb R^d$, let $\lambda_1>\dots>\lambda_d$ be the eigenvalues of $A$, with associated orthonormal eigenvectors $v_1,\dots,v_d$. Then $x\in E_{A,b}$ if and only if $x=x_0+x_1$ where $x_0=\frac{A^{-1}b}{2}$ and 

\[x_1=\frac{\sqrt{b^{\top}A^{-1}b}}{2} \sum_{i=1}^d \frac{c_i v_i}{\sqrt{\lambda_i}} \]

for $(c_1,\dots,c_d)\in\mathbb S^{d-1}.$

\end{lemma}

\begin{proof}

Writing $x=x_0+x_1$, we derive 

\begin{align}
    x_1^{\top} Ax_1+x_1^{\top}b+\frac{b^{\top}A^{-1}b}{4}&=x_1^{\top} Ax_1+2x_1^{\top}Ax_0 +x_0Ax_0\nonumber\\
    &=x^{\top}Ax\nonumber\\
    &=b^{\top}(x_1+x_0) \label{eq:Eab}\\
    &=b^{\top}x_1+\frac{b^{\top}A^{-1}b}{2}.\nonumber
\end{align}

Since we used the condition $x\in E_{A,b}$ only in reaching line~\eqref{eq:Eab}, the initial and final expressions are equal if and only if $x\in E_{A,b}$. It follows that $x=x_0+x_1\in E_{A,b}$ if and only if

\[ x_1^{\top} Ax_1=\frac{b^{\top}A^{-1}b}{4}.\]

This easily leads to the parametrization given and concludes the proof.
\end{proof}

\begin{lemma}\label{lem:induct}

Let $M^k\subseteq\mathbb R^d$ be a compact manifold of dimension $k\geq 1$ passing through the origin, and such that $T_0(M^k)\subsetneq b^{\perp}$. Then for all but a measure-zero set of positive-definite matrices $A$, the ellipsoid $E_{A,b}$ is transversal to $M^k$.

\end{lemma}

\begin{proof}[Proof of Lemma~\ref{lem:induct}]

Fixing $A_0\in \mathcal A^{\circ}$, Proposotion~\ref{prop:serre} ensures the existence of an open neighborhood $U_{A_0}\subseteq \mathcal A^{\circ}$ of $A_0$ on which the eigenvalues $\lambda_1(A)>\lambda_2(A)>\dots>\lambda_d(A)$ and associated orthonormal eigenvectors $v_1(A),\dots,v_d(A)$ are defined smoothly on all $A\in U_{A_0}$. Define $F:U_A\times \mathbb S^{d-1}\to \mathbb R^d$ by:

\[F(A,(c_1,\dots,c_d))= \frac{A^{-1}b}{2}+\frac{\sqrt{b^{\top}A^{-1}b}}{2}\sum_{i=1}^{d}\frac{ c_iv_i(A)}{\sqrt{\lambda_i(A)}}.\]

Lemma~\ref{lem:ellipse} implies that for each fixed $A$ we obtain a diffeomorphism $F_A:\mathbb S^{n-1}\to E_{A,b}.$ Moreover, $F$ is smooth by construction. We claim that $F$ and $M^k$ are transversal. To check this, we must verify that for any $z=F(A,c)\in M^k$, it holds that 
\[\text{Image}\left(dF\circ T_{F^{-1}(z)}(U_{A_0}\times \mathbb S^{N-1})\right)+T_z(M^k)=\mathbb R^d.\] 

First, recall that fixing $A=A_0$, the map $F_{A_0}:\mathbb S^{n-1}\to E_{A_0,b}$ is a diffeomorphism. Therefore \[\text{Image}\left(dF\circ T_{F^{-1}(z)}(U_{A_0}\times \mathbb S^{N-1})\right)\] contains the tangent space $T_z(E_{A,b})=b^{\perp}$ of $E_{A,b}$ at $z$. When $z=0$ is the origin, the assumption $T_0(M^k)\subsetneq b^{\perp}$ implies 

\[\dim\left(\text{Image}\left(dF\circ T_{F^{-1}(z)}(U_{A_0}\times \mathbb S^{N-1})\right)+T_z(M^k)\right)\geq \dim(b^{\perp})+1=d\]

and the claim follows. Supposing for the remainder of the proof that $z\neq 0$ is not the zero vector, we claim that in fact \[\text{Image}\left(dF\circ T_{F^{-1}(z)}(U_{A_0}\times \mathbb S^{N-1})\right)+T_z(M^k)=\mathbb R^d,\] i.e. the tangent space of $M^k$ is unnecessary.  Indeed fixing $c\in\mathbb S^{N-1}$, we may vary $A\in U_A$ along the path $\gamma_A(t)= \frac{A}{t}$ for $t\in (1-\varepsilon,1+\varepsilon)$. It is not difficult to see directly that

\[F(tA,c)=tF(A,c).\]

Therefore differentiating $F$ along $\gamma$ gives

\[\frac{d}{dt}F(\gamma_A(t),(c_1,\dots,c_d))|_{t=1}=F(A,c).\]

This means $z\in \text{Image}\left(dF\circ T_{F^{-1}(z)}(U_{A_0}\times \mathbb S^{N-1})\right)+T_z(M^k)$. Because $E_{A,b}$ is strictly convex and passes through the origin, it follows that the tangent hyperplane to $E_{A,b}$ at $z$ does not pass through the origin, hence $z\notin T_z(E_{A,b})$. We have establish that $\text{Image}\left(dF\circ T_{F^{-1}(z)}(U_{A_0}\times \mathbb S^{N-1})\right)+T_z(M^k)$ contains both $T_z(E_{A,b})$ and $z\notin T_z(E_{A,b})$. Since $\dim\left(T_z(E_{A,b})\right)=d-1$ it follows that $\text{Image}\left(dF\circ T_{F^{-1}(z)}(U_{A_0}\times \mathbb S^{N-1})\right)+T_z(M^k)=\mathbb R^d$ for $z\neq 0$ as claimed. This shows the desired transversality for almost all $A\in U_{A_0}$. 

To extend the transversality to all of $\mathcal A_{M^k}^{\circ}$, we use the fact that $\mathcal A_{M^k}^{\circ}$ is $\sigma$-compact, i.e. is the union of countably many compact sets. In fact, any open subset of $\mathbb R^d$ is $\sigma$-compact. As a consequence, $\mathcal A_{M^k}^{\circ}$ is contained the union of countably many of open neighborhoods $U_{A_0}$ as constructed above. Since the set of matrices $A$ inside each $U_{A_0}$ violating the transversality statement has measure $0$, we conclude by countable additivity that the set of $A\in \mathcal A_{M^k}^{\circ}$ violating transversality has measure $0$ as well. This concludes the proof.
\end{proof}

\begin{lemma}\label{lem:almostall}

Fix linearly independent vectors $b_1,\dots,b_d \in \R^d$ and let $A_1,\dots,A_d$ be positive-definite matrices sampled independently from probability distributions on $\R^{\binom{d+1}{2}}$ which are absolutely continuous with respect to Lebesgue measure (i.e. which have a density). Then 

\[\left|\bigcap_{i=1}^d E_{A_i,b_i}\right|\geq 2\]

holds almost surely.

\end{lemma}

\begin{proof}

We proceed iteratively. For $k=d-1,\dots,1$ set

\[M^k=E_{A_1,b_1}\cap\dots\cap E_{A_{d-k},b_{d-k}}.\]

We show by induction that $M^k$ is almost surely a smooth compact manifold of dimension $k$. The base case $k=d-1$ is obvious, and for smaller $k$, we have 

\[M^{k}=M^{k+1}\cap E_{A,b}.\]

Lemma~\ref{lem:induct} combined with Lemma~\ref{prop:transversality0} now implies that $M^k$ is a smooth compact manifold of dimension $k$ almost surely, completing the inductive step.

Finally Proposition~\ref{prop:mod2} implies that assuming $M^1$ and $E_{A_d,b_d}$ are transverse (which holds with probability $1$), the number of intersection points $|M^1\cap E_{A_d,b_d}|$ is finite and even. Of course $|M^1\cap E_{A_d,b_d}|=\left|\cap_{i=1}^d E_{A_i,b_i}\right|$. Since $\cap_{i=1}^d E_{A_i,b_i}$ trivially contains the origin, it must also contain another point. This completes the proof.
\end{proof}

\begin{proof}[Proof of Theorem~\ref{thm:irm_main}]
\label{sec:irm_proof}

Given $A_1,\dots,A_d$, consider a sequence of $d$-tuples $\left(A_1^{(k)},\dots,A_d^{(k)}\right)_{k\geq 1}$ converging to $(A_1,\dots,A_d)$, i.e. satisfying

\[\lim_{k\to\infty} A_i^{(k)}=A_i \]

for each $i\in [d]$. Moreover assume that $\left|\bigcap_{i\in [d]}E_{A_i^{(k)},b_i}\right|\geq 2$ for each $k$; such a sequence certainly exists by Lemma~\ref{lem:almostall}. We also assume that the estimates 

\begin{equation}\label{eq:eigbound}\ell\leq \lambda_d(A_i^{(k)})\leq \lambda_{1}(A_i^{(k)})\leq L\end{equation}

hold for some positive constants $\ell,L$ where $\lambda_d,\lambda_1$ are the minimum and maximum eigenvalues. This last assumption is without loss of generality by restricting the values of $k$ to $k\geq k_0$ for suitably large $k_0$. For each $k$, choose a non-zero point

\[x_k\in \bigcap_{i\in [d]}E_{A_i^{(k)},b_i}\backslash \{0\}.\]

Such points exist because $|\bigcap_{i\in [d]}E_{A_i^{(k)},b_i}|\geq 2$. We claim the norms $|x_k|$ are bounded away from infinity, bounded away from zero, and that any sub-sequential limit $x_*$ satisfies 

\[x_*\in \bigcap_{i\in [d]}E_{A_i,b_i}.\]

It follows from the above claims that at least one sub-sequential limit $x_*$ exists (using the Bolzano-Weierstrass theorem) and that $|x_*|\neq 0$. Therefore the above claims suffice to finish the proof, and we now turn to their individual proofs.

First, since $x^{\top}_k A_i^{(k)}x_k\geq \lambda_{d}(A_i^{(k)})|x_k|^2\geq \ell|x_k|^2$ and $|b^{\top}_i x_k|\leq |b^{\top}_i|\cdot |x_k|$, it follows that $|x_k|\leq \frac{|b_1|}{\ell}$ for all $k$, so in particular these norms are bounded above. Next we show the values $|x_k|$ are bounded away from $0$. Suppose for sake of contradiction that $|x_{a_j}|\to 0$ along some subsequence $(a_j)_{j\geq 1}$. Then

\[\langle b_i,x_{a_j}\rangle=x_{a_j}^{\top}A_i^{(a_j)}x_{a_j}\leq L|x_{a_j}|^2=o(|x_{a_j}|).\]

Defining the rescaled unit vectors $\widehat{x}_{a_j}=\frac{x_{a_j}}{|x_{a_j}|}$, it follows that 

\[\lim_{j\to\infty}\langle b_i,\widehat{x}_{a_j}\rangle = 0\]

for each $i$. As the $\widehat{x}_{a_j}$ are unit vectors, the Bolzano-Weierstrass theorem guarantees existence of a subsequential limit $\widehat{x}_*$ which is also a unit vector. It follows $\langle b_i,\widehat{x}_*\rangle=0$ for all $i\in [d]$. However because the vectors $b_i$ are linearly independent, this implies $|\widehat{x}_*|=0$ which is a contradiction. We conclude that $|x_k|$ is bounded away from $0$.

Finally we show that any subsequential limit satisfies $x_*\in E_{A,b}$. With $b$ fixed, observe that the functions $g_{A,b}(x)=x^{\top}Ax-b^{\top}x$ are uniformly Lipschitz for $A$ obeying the eigenvalue bound~\eqref{eq:eigbound} and $|x|\leq \frac{|b_1|}{\ell}$. It follows that 

\[\lim_{k\to\infty} g_{A^{(k)}_i,b_i}(x_*)=\lim_{k\to\infty}g_{A^{(k)}_i,b_i}(x_k)=0. \]

Having established the three claims we conclude the proof of Theorem~\ref{thm:irm_main}.
\end{proof}

\section{Additional experimental details}
\label{sec:experiments_app}
For Noised MNIST dataset, for each class $c \in \{0, \dots, 9\}$, we first generative a class signature $x_c \in \R^{28} \sim N(0, 2.5 I_{28})$. For each of the $E=12$ groups, we generate a training spurious covariance $\Sigma_2^e = G_e G_e^\top$ and a test spurious covariance ${\Sigma_2^e}' = G_e' {G_e'}^\top$. The noise code for digit $c$ in training environment $e$ is drawn from $\cN(x_c, \Sigma_2^e)$. In test environment, the noise is drawn from $\cN(x_{c'}, {\Sigma_2^e}')$ for random label $c' \sim \textup{unif}\{0, \dots, 9\})$.

We use SGD optimizer for both datasets. The hyperparameters are the coefficients for coral penalty, orthonormal penalty, and irm penalty $\lambda_{coral}, \lambda_{on}, \lambda_{irm}$, and learning rate $lr$. For each algorithm in Figures~\ref{fig:1} and~\ref{fig:2}, we select penalization strengths from $\{0.1, 1, 10, 100\}$ and $lr$ from $\{0.1, 0.01, 0.001, 0.0001\}$ that achieves highest average test accuracy within $500$ epochs (for Gaussian dataset) and $400$ epochs (for Noised MNIST). Gaussian dataset has batch size 100 and Noised MNIST has batch size 1000 from each training environment.

The average test accuracies for each algorithm with error bars are shown in Figures~\ref{fig:1} and~\ref{fig:2}. We fix the datasets and use different random seeds for algorithmic randomness. Error bar indicates mean and standard deviation across 5 runs.

The MLP architecture in Figure~\ref{fig:2} is in Table~\ref{tab:MLP}:
\begin{table}[h]
\caption{MLP network architectures for Noised MNIST}
\label{tab:MLP}
\centering
\begin{tabular}{ |c|c|c|c|c|c|}
\hline
 Number of layers & 1 & 2 & 3 & 4 & 6 \\ \hline
 Layer widths & 24 & 96, 24 & 128, 50, 24 & 192, 96, 48, 24 & 400, 300, 200, 100, 50, 24  \\ 
 \hline
\end{tabular}
\end{table}

\begin{table}[h]
\caption{Matching features at 3 layers with identical widths does not have significant advantage over matching only at the last layer (CORAL).}
\label{tab:shrink}
\centering
\begin{tabular}{ |c|c|c|c|}
\hline
 Layer widths & 24 & 128, 50, 24 & 24, 24, 24  \\ 
 \hline
 ERM & $58.6 \pm 0.4$ & $56.0 \pm 0.6$ & $62.1 \pm 0.6$ \\
 \hline
 IRM & $59.0 \pm 0.2$ & $56.1 \pm 0.6$ & $62.3 \pm 1.0$ \\
 \hline
 CORAL (only match last layer) & $69.1 \pm 1.0$ & $65.2 \pm 1.0$ & $67.2 \pm 0.4$ \\
 \hline
 CORAL (match-disjoint) & $69.1\pm 1.0$ & $75.5 \pm 1.0$ & $70.6 \pm 0.9$ \\
 \hline
 CORAL (match-all) & $69.1\pm 1.0$ & $77.9 \pm 0.4$ & $70.4 \pm 0.9$ \\
 \hline
\end{tabular}
\end{table}

To answer (Q5), we compare performances of algorithms on a 3-layer MLP that does not shrink feature dimensions (right column) with those on a 3-layer MLP that does (middle column) and a 1-layer MLP (left column) in Table~\ref{tab:shrink}. Results show that without shrinking feature dimensions, matching at multiple layers does not improve over naive CORAL on a smaller architecture.

No run in any of our experiments take more than 10 minutes on a single GPU. MNIST dataset~\citep{mnist} is made available under the terms of the Creative Commons Attribution-Share Alike 3.0 license.

\section{A simple algorithm achieves O(1) environment complexity under Assumption~\ref{ass:new_cov}}
\label{sec:simple_app}

Intuitively, subtracting the label-conditional covariances of any two environments yields the subspace of spurious coordinates (the column subspace of $B \in \R^{d \times \ds}$, the right $\ds$ columns of $S$). Once we obtain the projection matrix onto this subspace denoted as $P_B$, we can transform all observations $(X_i^e, Y_i^e)$ to $({X_i^e}', Y_i^e)$ where ${X_i^e}'=(I-P_B) X_i^e$ is the projection of $X_i^e$ onto the orthogonal subspace of $B$. The transformed inputs have no signal in any spurious dimension, so the optimal classifier on $({X_i^e}', Y_i^e)_{i=1}^\infty$ from any environment $e$ is the invariant predictor $w^*$.

\begin{algorithm}
\caption{A simple algorithm under Assumption~\ref{ass:new_cov} \label{alg_simple}}
\begin{algorithmic}[1]
\Require Invariant feature dimension $r$, spurious feature dimensions $\ds$, 2 training environments with infinite samples $\{(X_i^e,Y_i^e)\}_{i=1}^\infty \sim P_e$, $\{(X_i^{e'},Y_i^{e'})\}_{i=1}^\infty \sim P_{e'}$.
\State Subtract the covariances of class 1 examples between the two environments

$B = Cov_e[X | Y=1] - Cov_{e'}[X | Y=1]$.
\State Perform SVD on $B = Q \Gamma Q^\top$ to get orthonormal $Q \in \R^{d \times d_s}$ and diagonal $\Gamma \in \R^{d_s \times d_s}$. 
\State Project the mean of class 1 examples $\E[X|Y=1]$ to the orthogonal subspace of $B$,

$\mu' = (I-Q Q^\top) \E[X|Y=1]$.
\State Project the covariance of class 1 examples $\Sigma' = (I-Q Q^\top) Cov_e[X | Y=1] (I-Q Q^\top)$.
\State Return classifier $\hat{w} = {\Sigma'}^{\dagger} \mu'$.
\end{algorithmic}
\end{algorithm}

This is formalized as Algorithm~\ref{alg_simple}, and the following theorem provides formal guarantees for the environment complexity of this algorithm:
\begin{theorem}
\label{thm:simple}
Under Assumption~\ref{ass:new_cov}, Algorithm~\ref{alg_simple} satisfies $\hat{w}=w^*$.
\end{theorem}

\begin{proof}
Define $A, B$ as the left $r$ and right $\ds$ columns of $S$.
The optimal output is characterized by \begin{align*}
    \begin{bmatrix}
        A^\top\\
        B^\top
        \end{bmatrix} w^* = \begin{bmatrix}
        \Sigma_1^{-1} \mu_1\\
        0
        \end{bmatrix} \iff A^\top w^* = \Sigma_1^{-1} \mu_1, B^\top w^* = 0 \iff \Sigma_1 A^\top  w^* = \mu_1, P_B w^* = 0.
\end{align*}
The algorithmic output $\hat{w}$ satisfies
\begin{align*}
    (I-P_B) A \Sigma_1 A^\top (I-P_B) \hat{w} = (I-P_B) A \mu_1, P_B \hat{w} = 0 \\ \implies (I-P_B) A \Sigma_1 A^\top  \hat{w} = (I-P_B) A \mu_1 , P_B \hat{w} = 0
\end{align*}
Multiplying the first equation on the RHS by its pseudo-inverse, we get:
\begin{align*}
    (A^\top (I-P_B) A)^{-1} A^\top (I-P_B) A \Sigma_1 A^\top  \hat{w} =  (A^\top (I-P_B) A)^{-1} A^\top (I-P_B) A \mu_1 
    \implies \Sigma_1 A^\top  \hat{w} = \mu_1.
\end{align*} Therefore $\hat{w} = w^*$.
\end{proof}
\end{document}